\documentclass[10pt]{article}

\usepackage[paperwidth=7in, paperheight=10in, textwidth=5.25in, textheight=8.2in]{geometry}

\usepackage{subfig} 
\usepackage{float}
\usepackage{amsmath}
\usepackage{amsthm}
\usepackage{amssymb}
\usepackage{natbib}
\usepackage{graphicx}
\usepackage{booktabs}
\usepackage{mathtools}
\usepackage{xr,tikz}
\usepackage{enumitem}
\usepackage{nicefrac}
\usepackage{forloop}
\usepackage{hyperref}
\hypersetup{
	colorlinks = true,
	linkcolor  = blue!80!black,
	citecolor  = blue!80!black,
	urlcolor   = blue!80!black,
}
\usepackage{url}
\usepackage{bbm} 
\usepackage{authblk}
\usepackage[capitalise]{cleveref}
\usepackage{environ}
\usepackage{bm,xspace}
\usepackage{microtype}
\usepackage{mathtools}

\newcommand{\pr}{\text{Pr}}
\newcommand{\fa}{f_{avg}}
\newcommand{\ha}{h_{avg}}
\newcommand{\dom}{\textrm{dom}}
\newcommand{\edgestates}{Q}
\newcommand{\edgestate}{q}
\newcommand{\ground}{\Omega}
\PassOptionsToPackage{noend}{algorithmic}
\usepackage{amssymb}

\let\emptyset\varnothing

\DeclareMathOperator*{\argmax}{arg\,max}
\newcommand{\InDegree}{{d_{\mathrm{in}}}}

\usepackage[capitalise]{cleveref}
\usepackage{eqparbox,array}
\usepackage{algorithm}
\usepackage[noend]{algorithmic}

\let\emptyset\varnothing

\makeatother

\newcommand{\defcal}[1]{\expandafter\newcommand\csname c#1\endcsname{{\mathcal{#1}}}}
\newcommand{\defbb}[1]{\expandafter\newcommand\csname b#1\endcsname{{\mathbb{#1}}}}
\newcounter{calBbCounter}
\forLoop{1}{26}{calBbCounter}{
	\edef\letter{\Alph{calBbCounter}}
	\expandafter\defcal\letter
	\expandafter\defbb\letter
}

\usepackage{amsthm}
\makeatletter
\newtheorem*{rep@theorem}{\rep@title}
\newcommand{\newreptheorem}[2]{%
	\newenvironment{rep#1}[1]{%
		\def\rep@title{#2 \ref{##1}}%
		\begin{rep@theorem}}%
		{\end{rep@theorem}}}
\makeatother

\newtheorem{theorem}{Theorem}
\newreptheorem{theorem}{Theorem}

\newtheorem{definition}{Definition}
\newreptheorem{definition}{Definition}

\newtheorem{lemma}{Lemma}

\newtheorem{corollary}{Corollary}

\usepackage[capitalise]{cleveref}
\crefname{corollary}{Corollary}{Corollaries}
\newtheorem{observation}{Observation}

\title{Adaptive Sequence Submodularity}

\author[1]{Marko Mitrovic}
 
\author[1]{Ehsan Kazemi\protect}
\author[2]{Moran Feldman}
\author[3]{\\Andreas Krause}
\author[1]{Amin Karbasi}
\affil[1]{Yale University}
\affil[2]{Open University of Israel}
\affil[3]{ETH Z\"urich}

\date{}
\begin{document}

\maketitle

\begin{abstract}
	In many machine learning applications, one needs to interactively select a sequence of items (e.g., recommending movies based on a user's feedback) or make sequential decisions in a certain order (e.g., guiding an agent through a series of states). Not only do sequences already pose a dauntingly large search space, but we must also take into account past observations, as well as the uncertainty of future outcomes. Without further structure, finding an optimal sequence is notoriously challenging, if not completely intractable. In this paper, we view the problem of adaptive and sequential decision making through the lens of submodularity and propose an adaptive greedy policy with strong theoretical guarantees. Additionally, to demonstrate the practical utility of our results, we run experiments on Amazon product recommendation and Wikipedia link prediction tasks.

\end{abstract}

\section{Introduction}

The machine learning community has long recognized the importance of both sequential and adaptive decision making. The study of sequences has led to novel neural architectures such as LSTMs \citep{hochreiter97}, which have been used in a variety of applications ranging from machine translation \citep{sutskever14} to image captioning \citep{vinyals15}. Similarly, the study of adaptivity has led to the establishment of some of the most popular subfields of machine learning including active learning \citep{settles12} and reinforcement learning \citep{sutton18}. 

In this paper, we consider the optimization of problems where both sequences and adaptivity are integral part of the process. More specifically, we focus on problems that can be modeled as selecting a sequence of items, where each of these items takes on some (initially unknown) state. The idea is that the value of any sequence depends not only on the items selected and the order of these items but also on the states of these items.

Consider recommender systems as a running example. To start, the order in which we recommend items can be just as important as the items themselves. For instance, if we believe that a user will enjoy the Lord of the Rings franchise, it is vital that we recommend the movies in the proper order. If we suggest that the user watches the final installment first, she may end up completely unsatisfied with an otherwise excellent recommendation. Furthermore, whether it is explicit feedback (such as rating a movie on Netflix) or implicit feedback (such as clicking/not clicking on an advertisement), most recommender systems are constantly interacting with and adapting to each user. It is this feedback that allows us to learn about the states of items we have already selected, as well as make inferences about the states of items we have not selected yet.

Unfortunately, the expressive modeling power of sequences and adaptivity comes at a cost. Not only does optimizing over sequences instead of sets exponentially increase the size of the search space, but adaptivity also necessitates a probabilistic approach that further complicates the problem. Without further assumptions, even approximate optimization is infeasible. As a result, we address this challenge from the perspective of \textit{submodularity}, an intuitive diminishing returns condition that appears in a broad scope of different areas, but still provides enough structure to make the problem tractable.

Research on submodularity, which itself has been a burgeoning field in recent years, has seen comparatively little focus on sequences and adaptivity. This is especially surprising because many problems that are commonly modeled under the framework of submodularity, such as recommender systems \citep{GabillonKWEM2013, yue11} and crowd teaching \citep{singla2014near}, stand to benefit greatly from these concepts.

While the lion's share of existing research in submodularity has focused on \textit{sets}, a few recent lines of work extend the concept of submodularity to \textit{sequences}. \citet{tschiatschek17} were the first to consider \textit{sequence submodularity} in the general graph-based setting that we will follow in this paper.  
They presented an algorithm with theoretical guarantees for directed acyclic graphs, while \citet{mitrovic18a} developed a more comprehensive algorithm that provides theoretical guarantees for general hypergraphs. 

In their experiments, both of these works showed that modeling the problem as sequence submodular (as opposed to set submodular) gave noticeable improvements. Their applications could benefit even further from the aforementioned notions of adaptivity, but the existing theory behind sequence submodularity simply cannot model the problems in this way. While adaptive \textit{set} submodularity has been studied extensively \citep{golovin2011adaptive, chen13adaptive, gotovos2015non, fujii2019beyond}, these approaches still fail to capture order dependencies. 

\citet{alaei10} and \citet{zhang16} also consider sequence submodularity (called string-submodularity in some works), but they use a different definition, which is based on subsequences instead of graphs. 
On the other hand, \citet{li17} have considered the interaction of graphs and submodularity, but not in the context of sequences. 

\paragraph{Other Related Work} Amongst many other applications, submodularity has also been used for variable selection~\citep{krause05near}, data summarization~\citep{mirzasoleiman13distributed,lin2011class,kirchhoff2014submodularity}, 
sensor placement \citep{krause07jmlr}, neural network interpretability~\citep{elenberg17}, network inference~\citep{gomez10}, and influence maximization in social networks~\citep{kempe03}. 
Submodularity has also been studied extensively in a wide variety of settings, including distributed and scalable optimization~\citep{kumar2013fast,mirzasoleiman13distributed,barbosa2015power,mirrokni2015randomized,fahrbach2018submodular, balkanski2018exponential, balkanski2018adaptive, ene2018submodular}, streaming algorithms~\citep{krause2010budgeted, badanidiyuru2014streaming, chakrabarti2014submodular, buchbinder2015online, mitrovic18b, feldman2018do, norouzifard2018beyond, kazemi2019submodular}, robust optimization~\citep{krause2008robust, bogunovic2017robust, tzoumas2017resilient, kazemi2018scalable, staib2017robust},
 weak submodularity~\citep{das2011submodular,elenberg2016restricted, elenberg2017streaming, khanna2017scalable}, and continuous submodularity~\citep{wolsey82,bach2015, hassani2017gradient, staib2018distributionally, bai2018submodular}.

\noindent \paragraph{Our Contributions} The main contributions of our paper are presented in the following sections:
\begin{itemize}
\item In \cref{sec:ass}, we introduce our framework of \textit{adaptive sequence submodularity}, which brings tractability to problems that include both sequences and adaptivity. 
\item In \cref{sec:theory}, we present our algorithm for adaptive sequence submodular maximization. We present theoretical guarantees for our approach and we elaborate on the necessity of our novel proof techniques. We also show that these techniques simultaneously improve the state-of-the-art bounds for the problem of sequence submodularity by a factor of $\frac{e}{e-1}$. Furthermore, we argue that any approximation guarantee must depend on the structure of the underlying graph unless the exponential time hypothesis is false.
\item In \cref{sec:experiments}, we use datasets from Amazon and Wikipedia to compare our algorithm against existing sequence submodular baselines, as well as state-of-the-art deep learning-based approaches.
\end{itemize}
\section{Adaptive Sequence Submodularity} \label{sec:ass}

As discussed above, sequences and adaptivity are an integral part of many real-world problems. This means that many real-world problems can be modeled as selecting a sequence $\sigma$ of items from a ground set $V$, where each of these items takes on some (initially unknown) state $o \in O$. A particular mapping of items to states is known as a \textbf{realization} $\phi$, and we assume there is some unknown distribution $p(\phi)$ that governs these states.

For example in movie recommendation, the set of all movies is our ground set $V$ and our goal is to select a sequence of movies that a particular user will enjoy. If we recommend a movie $v_i \in V$ and the user likes it, we place $v_i$ in state 1 (i.e. $o_i = 1$). If not, we put it into state 0. Naturally, the value of a movie should be higher if the user liked it, and lower if she did not.

Formally, we want to select a sequence $\sigma$ that maximizes $f(\sigma,\phi)$, where $f(\sigma,\phi)$ is the value of sequence $\sigma$ under realization $\phi$. However, $\phi$ is initially unknown to us and the state of each item in the sequence is revealed to us only after we select it. In fact, even if we knew $\phi$ perfectly, the set of all sequences poses an intractably large search space. From an optimization perspective, this problem is hopeless without further structural assumptions.  

Our first step towards taming this problem is to follow the work of \citet{tschiatschek17} and assume that the value of a sequence can be defined using a graph. Concretely, we have a directed graph $G = (V,E)$, where each item in our ground set is represented as a vertex $v \in V$, and the edges encode the additional value intrinsic to picking certain items in certain orders. Mathematically, selecting a sequence of items $\sigma$ will induce a set of edges $E(\sigma)$:
\begin{equation*}
E(\sigma) = \big\{ (\sigma_i, \sigma_j) \mid (\sigma_i, \sigma_j) \in E, i \leq j \big\}.
\end{equation*}
For example, consider the graph in Figure \ref{seq1} 
and consider the sequence $\sigma_A = [F,T]$ where the user watched The Fellowship of the Ring, and then The Two Towers, as well as the sequence $\sigma_B = [T,F]$ where the user watched the same two movies but in the opposite order. 
\begin{align*}
E(\sigma_A) &= E\big( [F,T] \big) = \big\{ (F,F),(T,T),(F,T) \big\}\\
E(\sigma_B) &= E\big( [T,F] \big) = \big\{ (T,T),(F,F) \big\}
\end{align*}

Using the self-loops, this graph encodes the fact that there is certainly some intrinsic value to watching these movies regardless of the order. On the other hand, the edge $(F,T)$ encodes the fact that watching The Fellowship of the Ring before The Two Towers will bring additional value to the viewer, and this edge is only induced if the movies appear in the correct order in the sequence.

\begin{figure*}[t]
\centering
\subfloat[]{\includegraphics[height=0.89in]{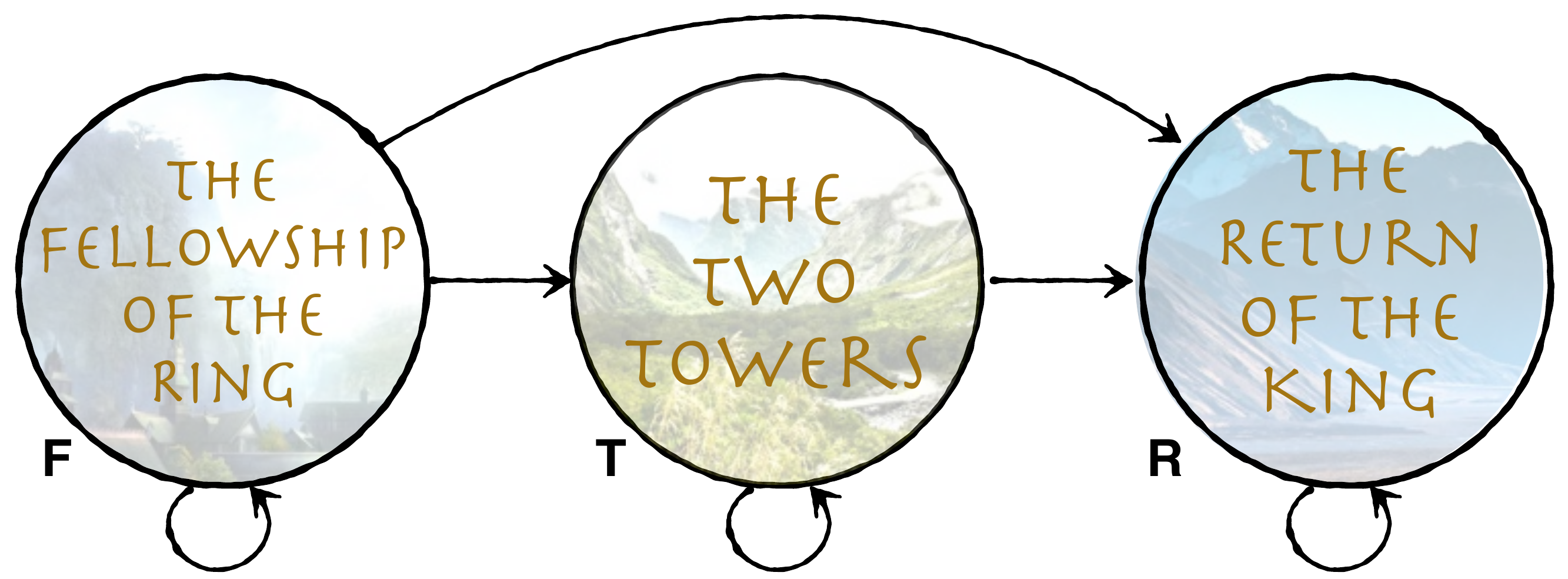}\label{seq1}}  \hspace{0.3in}
\subfloat[]{\includegraphics[height=0.89in]{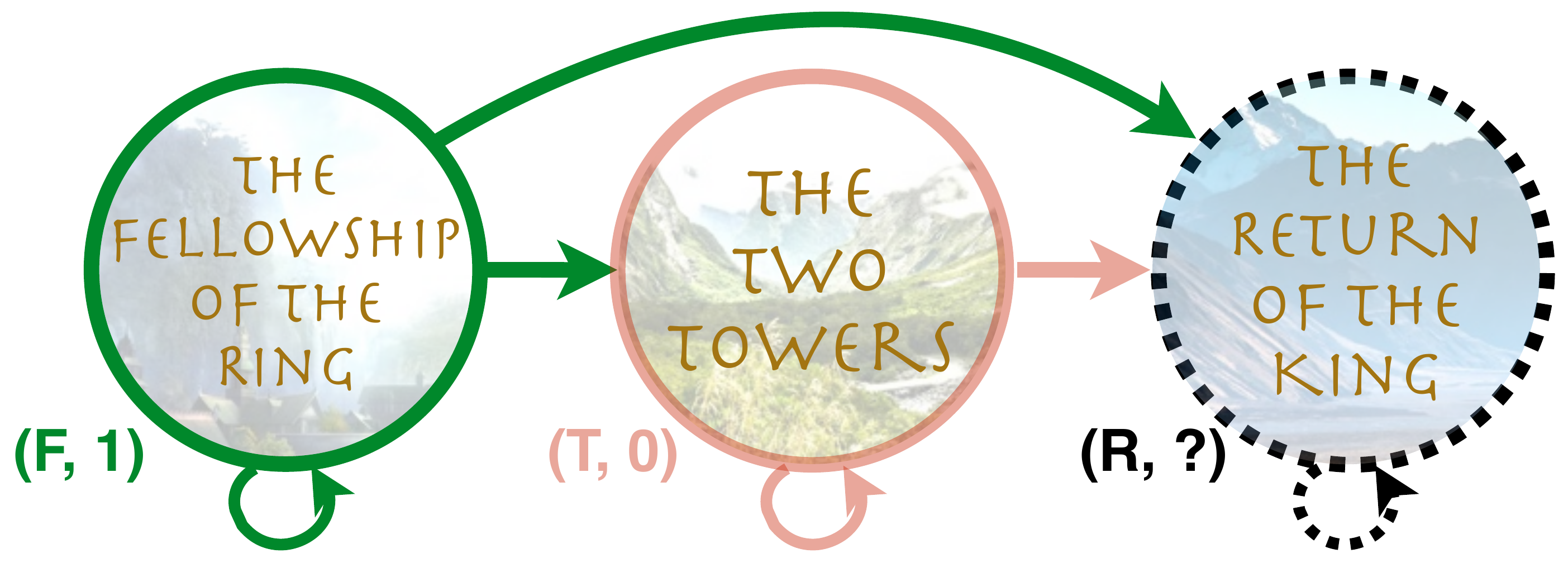}\label{seqA1}} 
\caption{(a) shows an underlying graph for a movie recommendation problem. The vertices are movies and edges denote the additional value of watching certain movies in certain orders. (b) extends this to the adaptive case, where both the vertices and the edges take on a state. The user has reported that she liked the Fellowship of the Ring (so it is placed in state 1), but she did not like The Two Towers (so it is placed in state 0). The state of the last movie is still unknown. In this example, the state of an edge is equal to the state of its starting vertex.}
\label{fig:examples}
\end{figure*}

With this graph based set-up, however, we run into issues when it comes to adaptivity. In particular, the states of items naturally translate to states for the vertices, but it is not clear how to extend adaptivity to the \textit{edges}. We tackle this challenge by assigning a state $\edgestate \in \edgestates$ to each edge strictly as a function of the states of its endpoints. That is, similarly to how a sequence $\sigma$ induces a set of edges $E(\sigma)$, a realization $\phi$ for the states of the vertices induces a realization $\phi^E$ for the states of the edges. As we will discuss later, the analysis for this approach will necessitate some novel proof techniques, but the resulting framework is very flexible and it allows us to fully redefine the adaptive sequence problem in terms of the underlying graph:
\begin{equation*}
f(\sigma,\phi) = h \big( E(\sigma), \phi^E \big) \ \text{where $\sigma$ induces $E(\sigma)$ and $\phi$ induces $\phi^E$.}
\end{equation*}

The last necessary ingredient to bring tractability to this problem is submodularity. In particular, we will assume that $h \big( E(\sigma), \phi^E \big)$ is \textit{weakly adaptive set submodular}. This is a relaxed version of standard adaptive set submodularity that can model an even larger variety of problems, and it is a natural fit for the applications we consider in this paper.

In order to formally define weakly-adaptive submodularity, we need a bit more terminology. To start, we define a \textbf{partial realization} $\psi$ to be a mapping for only some subset of items (i.e., the states of the remaining items are unknown). For notational convenience, we define the domain of $\psi$, denoted $dom(\psi)$, to be the list of items $v$ for which the state of $v$ is known. We say that $\psi$ is a \textbf{subrealization} of $\psi'$, denoted $\psi \subseteq \psi'$, if $dom(\psi) \subseteq dom(\psi')$ and they are equal everywhere in the domain of $\psi$. Intuitively, if $\psi \subseteq \psi'$, then $\psi'$ has all the same information as $\psi$, and potentially more. 

Given a partial realization $\psi$, we define the marginal gain of a set $A$ as
\begin{equation*}
\Delta(A \mid \psi) = \mathbb{E} \Big[ h \big(dom(\psi) \cup A, \phi \big) - h \big(dom(\psi),\phi \big) \mid \psi \Big],
\end{equation*}
where the expectation is taken over all the full realizations $\phi$ such that $\psi \subseteq \phi$. In other words, we condition on the states given by the partial realization $\psi$, and then we take the expectation across all possibilities for the remaining states.

\newcommand{\defWeakSetAdaptive}{%
	A function $h : 2^{E} \times \edgestates^{E} \to \mathbb{R}_{\ge 0}$ is \textbf{weakly adaptive set submodular}  with parameter $\gamma$ 
	if for all sets $A \subseteq E$ and for all 
	$\psi \subseteq \psi'$ we have:
	\begin{align*}
	\Delta(A \mid \psi') \leq \frac{1}{\gamma} \cdot \sum_{e \in A} \Delta(e \mid \psi).
	\end{align*}
}
\begin{definition} \label{def:WeakSetAdaptive}
\defWeakSetAdaptive
\end{definition}

This notion is a natural generalization of weak submodular functions \citep{das2011submodular} to adaptivity. The primary difference is that we condition on subrealizations instead of just sets because we need to account for the states of items. Note that in the context of this paper $h$ is a function on the edges, so we will condition on subrealizations of the edges $\psi^E$. However, these concepts apply more generally to functions on any set and state spaces, so we use $\psi$ in the formal definitions.

\newcommand{\defMonotoneAdaptive}
{
A function $h : 2^{E} \times \edgestates^{E} \to \mathbb{R}_{\ge 0}$ is \textbf{adaptive monotone}  
if $\Delta(e \mid \psi) \geq 0$ for all partial realizations $\psi$. That is, the conditional expected marginal benefit of any element is non-negative.
}
\begin{definition} \label{def:MonotoneAdaptive}
\defMonotoneAdaptive
\end{definition}

Figure \ref{seqA1} is designed to help clarify these concepts. It includes the same graph as Figure \ref{seq1}, but now we can receive feedback from the user. If we recommend a movie and the user likes it, we put the corresponding vertex in state 1 (green in the image). Otherwise, we put the vertex in state 0 (red in the image). Vertices whose states are still unknown are denoted by a dotted black line.

Next, in our example, we need to define a state for each edge in terms of the states of its endpoints. 
In this case, we will define the state of each edge to be equal to the state of its start point. 
In Figure \ref{seqA1}, the user liked The Fellowship of the Ring, which puts edges $(F,F)$, $(F,T)$, and $(F,R)$ in state 1 (green). She did not like The Two Towers, so edges $(T,T)$ and $(T,R)$ are in state 0 (red), and we do not know the state for The Return of the King, so the state of $(R,R)$ is also unknown. We call this partial realization $\psi_1$ for the vertices, and the induced partial realization for the edges $\psi_1^E$.

Suppose our function $h$ counts all induced edges that are in state 1. Furthermore, let us simply assume that any unknown vertex is equally likely to be in state 0 or state 1. This means that the self-loop $(R,R)$ is also equally likely to be in either state 0 or state 1. Therefore, $\Delta \big( (R,R) \mid \psi_1^E \big) =  \frac{1}{2} \times 0 + \frac{1}{2} \times 1 = \frac{1}{2}$. 

On the other hand, consider the edge $(F,R)$. Under $\psi_1$, we know $F$ is in state 1, which means $(F,R)$ is also in state 1, and thus, $\Delta \big( (F,R) \mid \psi_1^E \big) =  1$. However, if we consider a subrealization $\psi_2 \subseteq \psi_1$ where we do not know the state of $F$, then it is equally likely to be in either state and $\Delta \big( (F,R) \mid \psi^E_2 \big) =  \frac{1}{2} \times 0 + \frac{1}{2} \times 1 = \frac{1}{2}$. Therefore, for this simple function we know that $\gamma \leq 0.5$.

\section{Adaptive Sequence-Greedy Policy and Theoretical Results} \label{sec:theory}
In this section, we introduce our Adaptive Sequence-Greedy policy and present its theoretical guarantees.
We first formally define \textbf{weakly adaptive sequence submodularity}.

\newcommand{\defWeakSeqAdaptive}
{
	A function $f(\sigma, \phi)$ defined over a graph $G(V,E)$ is \textbf{weakly adaptive sequence submodular} if $f(\sigma,\phi) = h \big( E(\sigma), \phi^E \big)$ where a sequence $\sigma$ of vertices in $V$ induces a set of edges $E(\sigma)$, realization $\phi$ induces $\phi^E$, and  the function $h$ is weakly adaptive set submodular. Note that if $h$ is adaptive monotone, then $f$ is also adaptive monotone.
}
\begin{definition} \label{def:WeakSeqAdaptive}
\defWeakSeqAdaptive
\end{definition}

Formally, a policy $\pi$ is an algorithm that builds a sequence of $k$ vertices by seeing which states have been observed at each step, then deciding which vertex should be chosen and observed next. If $\sigma_{\pi,\phi}$ is the sequence returned by policy $\pi$ under realization $\phi$, then we write the expected value of $\pi$ as: 
\begin{equation*}
f_{\text{avg}}(\pi) = \mathbb{E} \big[ f ( \sigma_{\pi,\phi},\phi ) \big] = \mathbb{E} \Big[ h\big( E(\sigma_{\pi,\phi}),\phi^E \big) \Big]
\end{equation*}
where again the expectation is taken over all possible realizations $\phi$.

Our Adaptive Sequence Greedy policy $\pi$ (Algorithm \ref{alg:sequence-greedy}) starts with an empty sequence $\sigma$. Throughout the policy, we define $\psi_\sigma$ to be the partial realization for the vertices in $\sigma$. In turn this gives us the partial realization $\psi_\sigma^E$ for the induced edges.

At each step, we define the valid set of edges $\cE$ 
to be the edges whose endpoint is not already in $\sigma$. The main idea of our policy is that, at each step, we select the valid edge $e \in \cE$ with the highest expected value $\Delta(e \mid \psi_\sigma^E)$.
For each such edge, the endpoints that are not already in the sequence $\sigma$ are concatenated ($\oplus$ means concatenate) to the end of $\sigma$, and their states are observed (updating $\psi_\sigma$).

\begin{algorithm}[htb!]
	\caption{Adaptive Sequence Greedy Policy $\pi$}\label{alg:sequence-greedy}
	\begin{algorithmic}[1]
		\STATE  {\bfseries Input:} Directed graph $G = (V, E)$, weakly adaptive sequence submodular $f(\sigma,\phi) = h \big( E(\sigma), \phi^E \big)$, and cardinality constraint $k$\;
		\STATE  Let $\sigma \gets ()$\;
		\WHILE{$|\sigma| \leq k - 2 $}
		{
			\STATE$\cE = \{e_{ij} \in E \mid v_j \notin \sigma \}$ \;
			\IF{$\cE \neq \emptyset$}
			{
				\STATE		$e_{ij} = \argmax_{e \in \cE} \Delta(e \mid \psi_\sigma^{E})$ \;
				\IF{$v_i = v_j$ \textbf{or} $v_i \in \sigma$}
				{
					\STATE $\sigma = \sigma \oplus  v_j$ and observe state of $v_j$ \;
				}
				\ELSE{
					\STATE $\sigma = \sigma \oplus v_i \oplus  v_j$ and observe states of $v_i,v_j$ \;
				}
				\ENDIF 
			}
			\ELSE{
				\STATE \textbf{break}\;
			}
			\ENDIF
		}
		\ENDWHILE
		\STATE {\bfseries Return} $\sigma$\;
	\end{algorithmic}
\end{algorithm}

\newcommand{\TheoryAdaptiveSeq}{%
	For adaptive monotone and weakly adaptive sequence submodular function $f$, the Adaptive Sequence Greedy policy $\pi$ represented by \cref{alg:sequence-greedy} achieves
\[ \fa(\pi) \geq \dfrac{\gamma}{2 d_{\textup{in}} + \gamma} \cdot \fa(\pi^*),\]
where $\gamma$ is the weakly adaptive submodularity parameter, $\pi^*$ is the policy with the highest expected value and $d_{\textup{in}}$ is the largest in-degree of the input graph $G$.}

\begin{theorem} \label{theory:adaptive-seq}
\TheoryAdaptiveSeq
\end{theorem}

As discussed by \citet{mitrovic18a}, using a hypergraph $H$ instead of a normal graph $G$ allows us to encode more intricate relationships between the items. For example, in Figure \ref{seq1}, the edges only encode pairwise relationships. However, there may be relationships between larger groups of items that we want to encode explicitly. For instance, if included, the value of a hyperedge $(F,T,R)$ in Figure \ref{seq1} would explicitly encode the value of watching The Fellowship of the Ring, followed by watching The Two Towers, and then concluding with The Return of the King.

We can also extend our policy to general hypergraphs (see \cref{alg:hyper-sequence-greedy} in \cref{sec:hypergraph}). \cref{theory:hyper-adaptive-seq} guarantees the performance of our proposed policy for hypergraphs.

\newcommand{\TheoryAdaptivehyper}{%
	For adaptive monotone and weakly adaptive sequence submodular function $f$, the policy $\pi'$ represented by \cref{alg:hyper-sequence-greedy} achieves
	\[ \fa(\pi') \geq \dfrac{\gamma}{r d_{\textup{in}} + \gamma} \cdot \fa(\pi^*),\]
	where $\gamma$ is the weakly adaptive submodularity parameter, $\pi^*$ is the policy with the highest expected value and $r$ is the size of the largest hyperedge in the input hypergraph.}

\begin{theorem}\label{theory:hyper-adaptive-seq}
\TheoryAdaptivehyper
\end{theorem}

In our proofs, we have to handle the sequential nature of picking items and the revelation of states in a combined setting. 
Unfortunately, the existing proof methods for sequence submodular maximization are not linear enough to allow for the use of the linearity of expectation that captures the stochasticity of the states.
For this reason, we develop a novel analysis technique to guarantee the performance of our algorithms.
Surprisingly, these new techniques improve the theoretical guarantees of the non-adaptive Sequence-Greedy and Hyper Sequence-Greedy \citep{mitrovic18a} by a factor of $\frac{e}{e-1}$.
Proofs for both theorems are given in \cref{proofs}.

\paragraph{General Unifying Framework} One more theoretical point we want to highlight is that weakly adaptive sequence submodularity provides a general unifying framework for a variety of common submodular settings including, adaptive submodularity, weak submodularity, sequence submodularity, and classical set submodularity. If we have $\gamma = 1$ and the state of all vertices is deterministic, then we have sequence submodularity. Conversely, if the vertex states are unknown, but our graph only has self-loops, then we have weakly adaptive set submodularity (and correspondingly adaptive set submodularity if $\gamma = 1$). Lastly, if we have a graph with only self-loops, full knowledge of all states, and $\gamma = 1$, then we recover the original setting of classical set submodularity.

\paragraph{Tightness of Theoretical Results}
We acknowledge that the constant factor approximation we present depends on the maximum in-degree. While ideally the theoretical bound would be completely independent of the structure of the graph, we argue here that such a dependence is likely necessary.

Indeed, getting a dependence better than $O(n^{\nicefrac{1}{4}})$ in the approximation factor (where $n$ is the total number of items) would improve the state-of-the-art algorithm for the very well-studied densest $k$ subgraph problem (DkS) \citep{kortsarz1993on, bhaskara2010detecting}. Moreover, if we could get an approximation that is completely independent of the structure of the graph, then the exponential time hypothesis would be proven false\footnote{If the exponential time hypothesis is true it would imply that P $\neq$ NP, but it is a stronger statement.}.
 In fact, even an almost polynomial approximation would break the exponential time hypothesis \citep{manurangsi2017almost}.  
Next, we formally state this hardness relationship.
The proof is given in \cref{app:hardness}.

\newcommand{\TheoryHardness}{%
Assuming the exponential time hypothesis is correct, there is no algorithm that approximates the optimal solution for the (adaptive) sequence submodular maximization problem within a $n^{1/(\log \log n)^c}$ factor, where $n$ is the total number of items and $c > 0$ is a universal constant  independent of $n$.
}
\begin{theorem} \label{theory:hardness}
\TheoryHardness
\end{theorem}

\section{Experimental Results} \label{sec:experiments}

\subsection{Amazon Product Recommendation} \label{sec:amazon}

Using the Amazon Video Games review dataset \citep{amazonReviews}, we consider the task of recommending products to users. In particular, given the first $g$ products that the user has purchased, we want to predict the next $k$ products that she will buy. Full experimental details are given in \cref{amazonAdditional}.\footnote{Dataset and code are available at \url{https://github.com/ehsankazemi/adaptiveSubseq}.}

We start by using the training data to build a graph $G = (V,E)$, where $V$ is the set of all products and $E$ is the set of edges between these products. The weight of each edge, $w_{ij}$, is defined to be the conditional probability of purchasing product $j$ given that the user has previously purchased product $i$. There are also self-loops with weight $w_{ii}$ that represent the fraction of users that purchased product $i$.

We define the state of each edge $(i,j)$ to be equal to the state of product $i$. The intuitive idea is that edge $(i,j)$ encodes the value of purchasing product $j$ after already having purchased product $i$. Therefore, if the user has definitely purchased $i$ (i.e., product $i$ is in state 1), then they should receive the full value of $w_{ij}$. On the other hand, if she has definitely not purchased $i$ (i.e., product $i$ is in state 0), then edge $(i,j)$ provides no value. Lastly, if the state of $i$ is unknown, then the expected gain of edge $(i,j)$ is discounted by $w_{ii}$, the value of the self-loop on $i$, which can be viewed as a simple estimate for the probability of the user purchasing product $i$. See Figure \ref{amazonGraph} for a small example.

We use a probabilistic coverage utility function as our monotone weakly-adaptive set submodular function $h$. Mathematically, 
\[
h(E_1) = \sum_{j \in V} \Big[ 1 - \prod_{(i,j) \in E_1} (1 - w_{ij}) \Big],
\]
where $E_1 \subseteq E$ is the subset of edges that are in state 1. 

We compare the performance of our Adaptive Sequence-Greedy policy against Sequence-Greedy from \citet{mitrovic18a}, the existing sequence submodularity baseline that does not consider states. To give further context for our results, we compare against Frequency, a naive baseline that ignores sequences and adaptivity and simply outputs the $k$ most popular products. 

We also compare against a set of deep learning-based approaches (see \cref{deepDetails} for full details). In particular, we implement adaptive and non-adaptive versions of both a regular Feed Forward Neural Network and an LSTM. The adaptive version will update its inputs after every prediction to reflect whether or not the user liked the recommendation. Conversely, the non-adaptive version will simply make $k$ predictions using just the original input.

\begin{figure*}[t]
	\centering
    \subfloat[]{\includegraphics[height=1.2in]{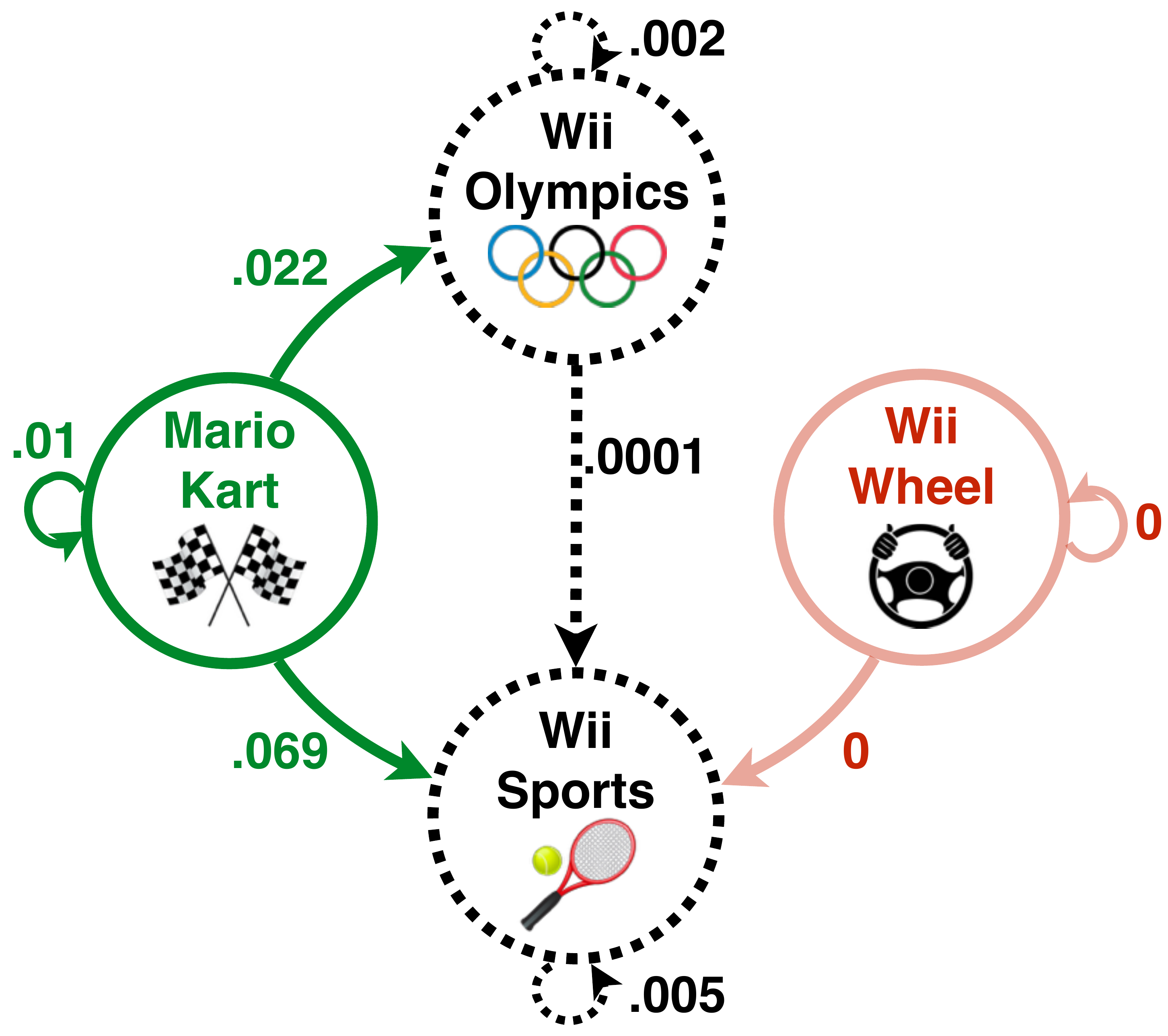}\label{amazonGraph}} \hspace{0.22in}
	\subfloat[]{\includegraphics[height=1.2in]{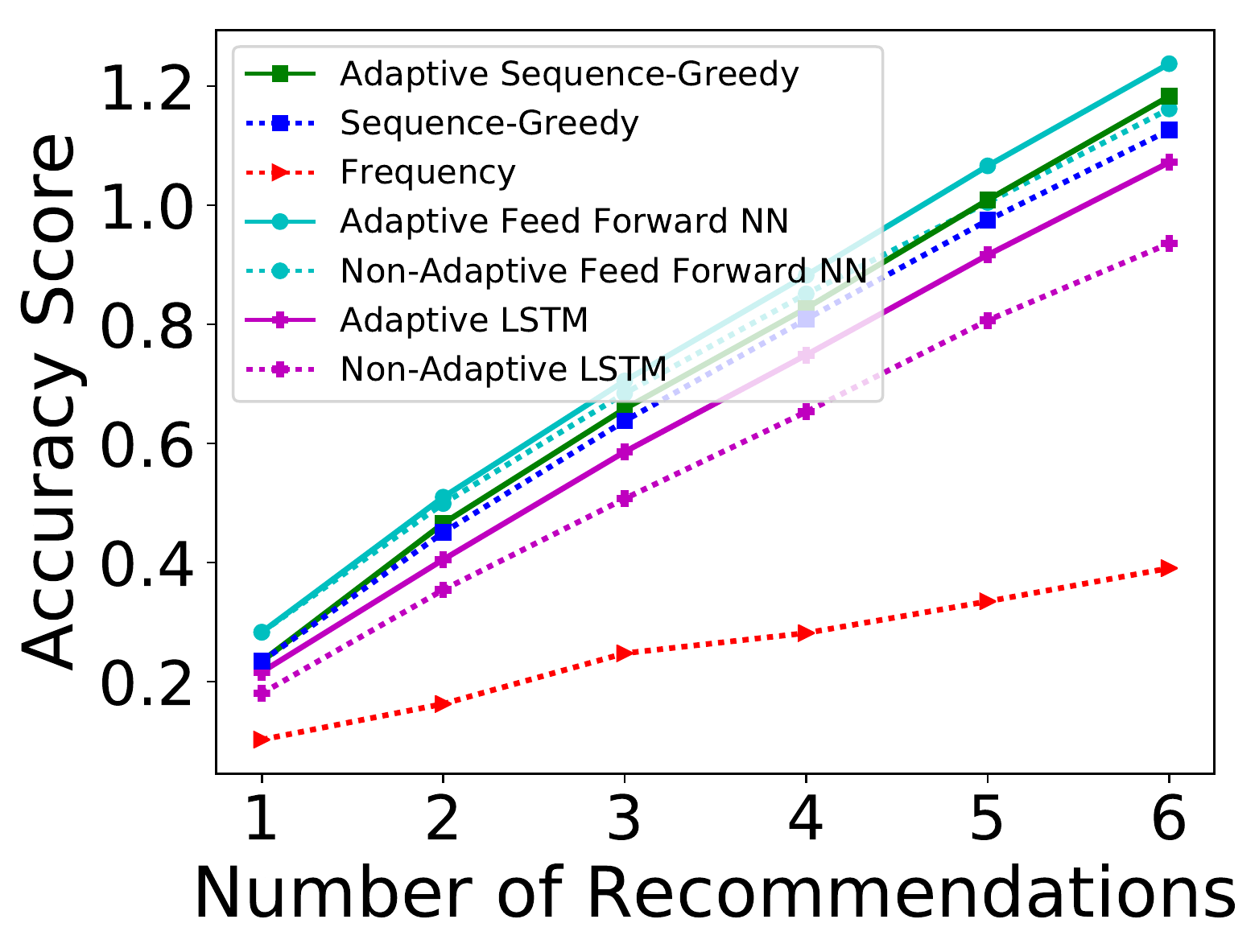}\label{amazon-80-acc}}
	\subfloat[]{\includegraphics[height=1.2in]{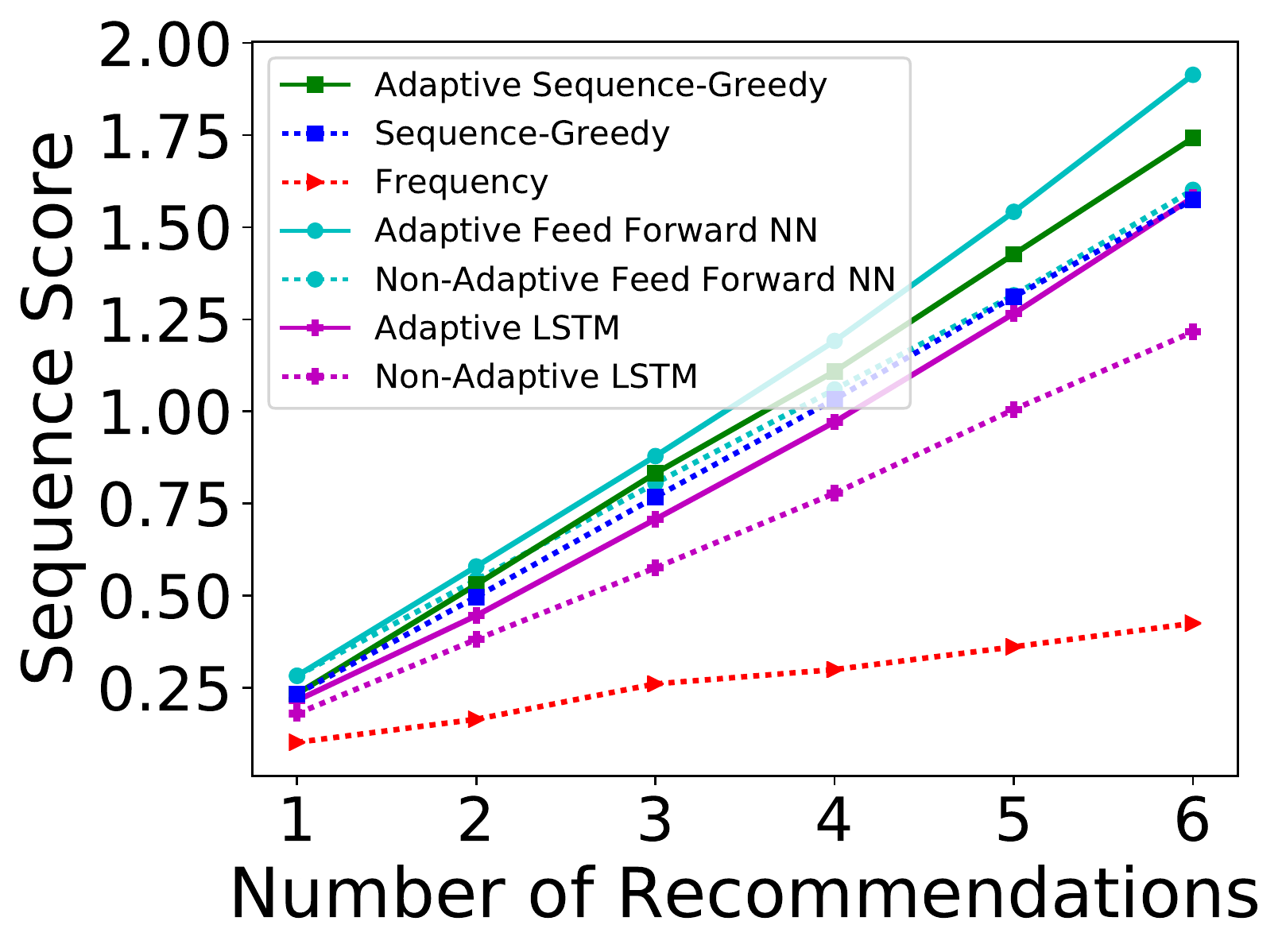}\label{amazon-80-seq}} \\
	\subfloat[]{\includegraphics[height=1.in]{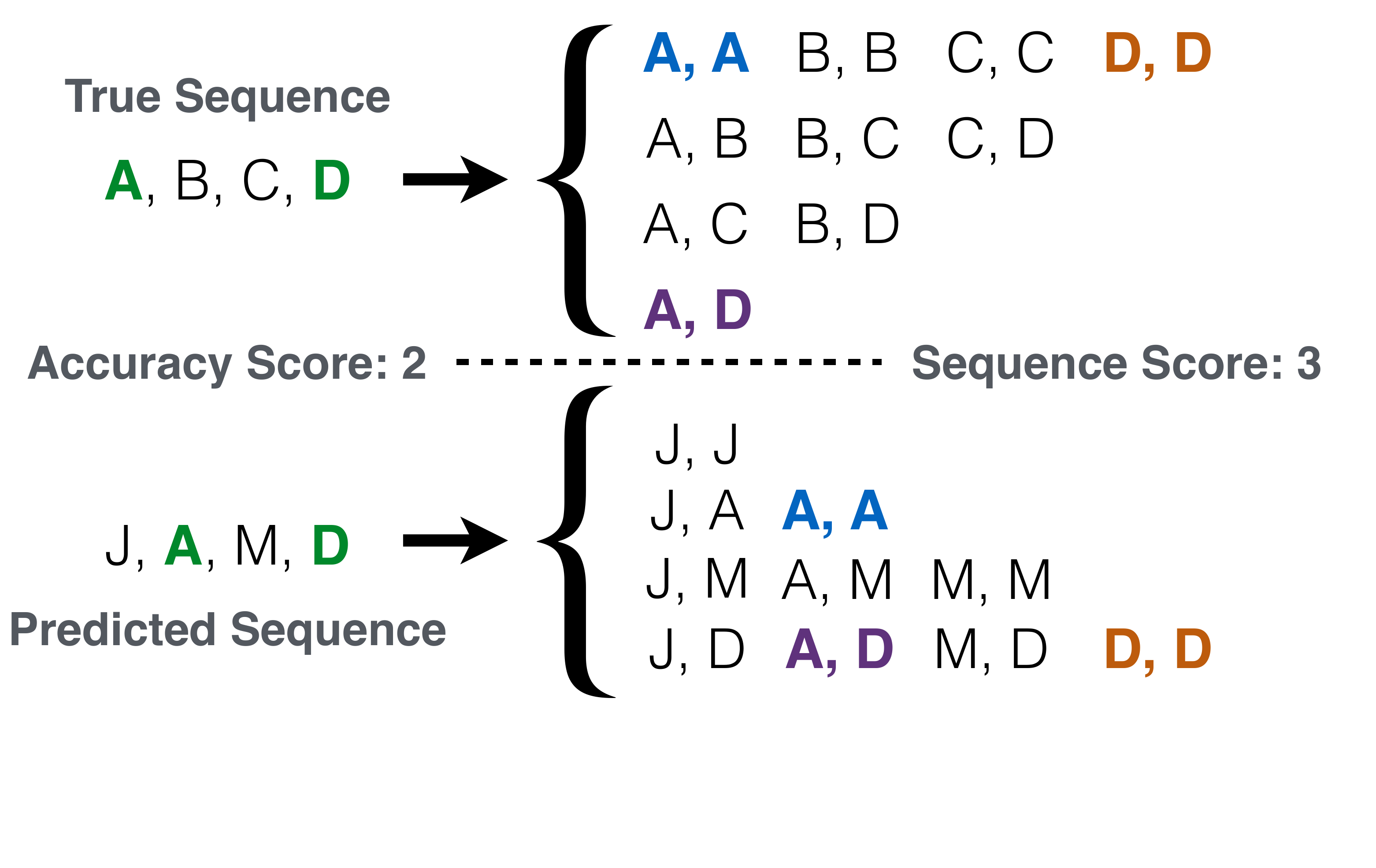}\label{scores}}
	\subfloat[]{\includegraphics[height=1.2in]{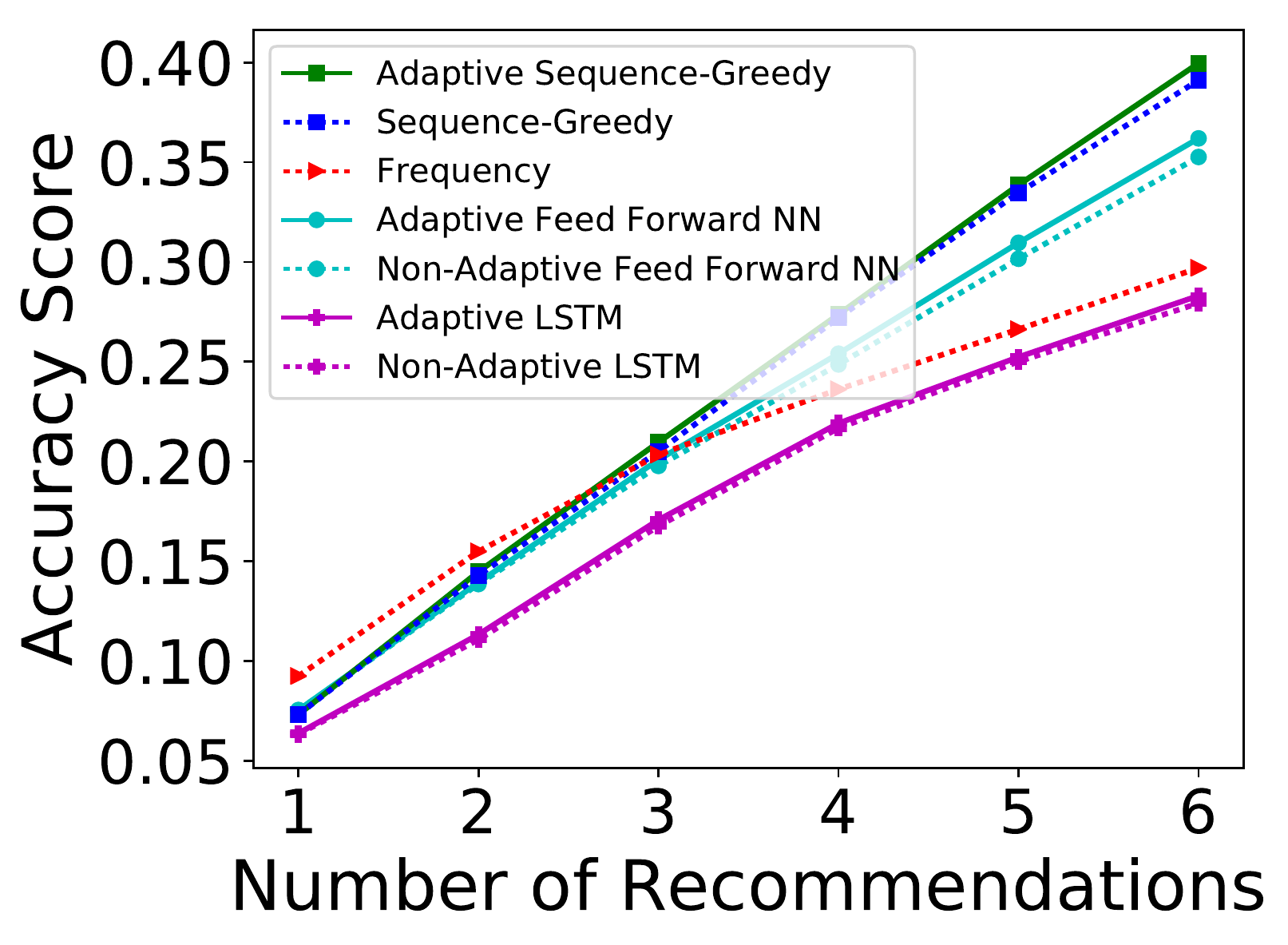}\label{amazon-1-acc}}
	\subfloat[]{\includegraphics[height=1.2in]{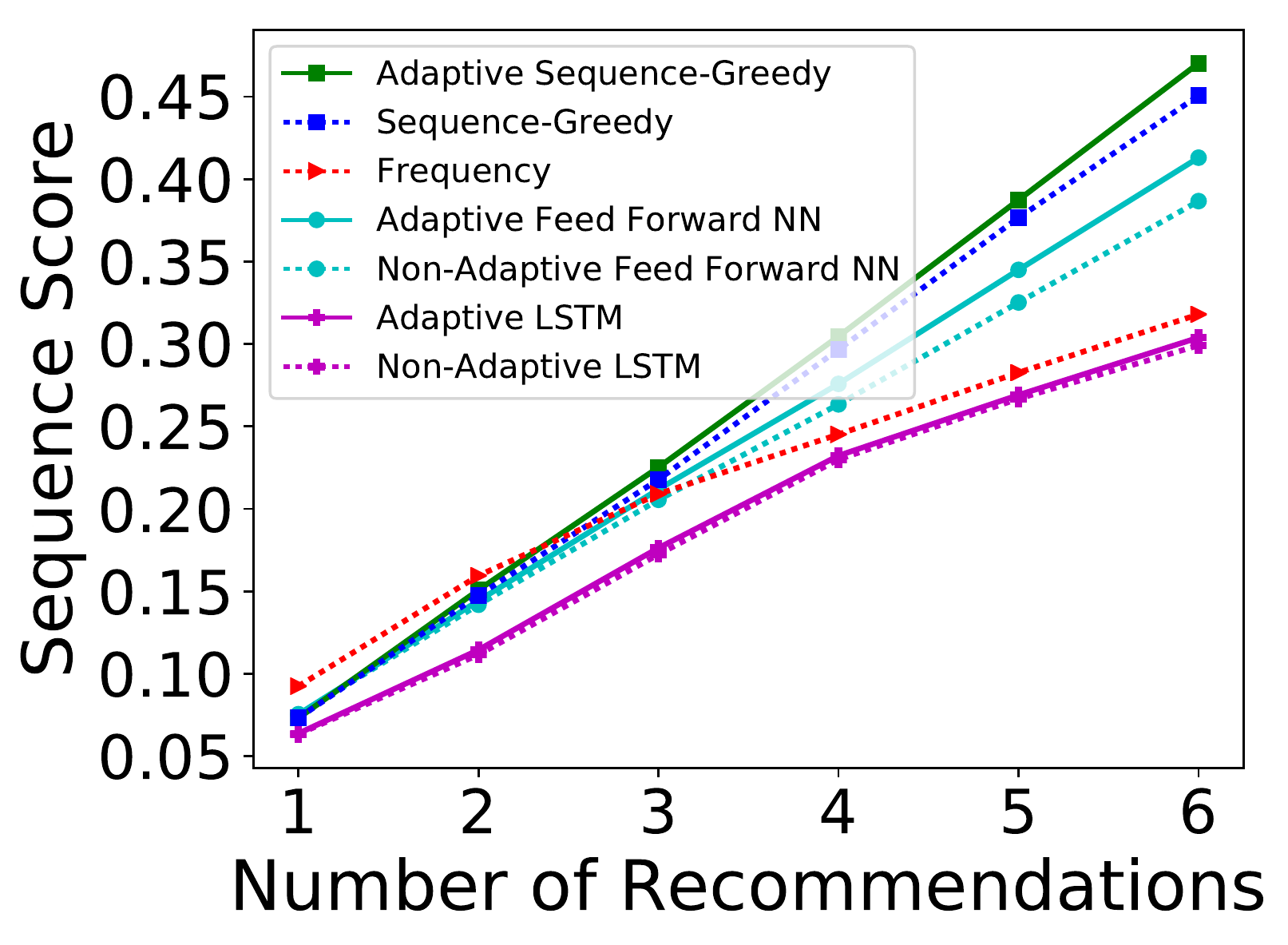}\label{amazon-1-seq}}
	\caption{(a) shows a small subset of the underlying graph with states for a particular user. 
		(b) and (c) show our results on the Amazon product recommendation task. In all these graphs, the number of given products $g$ is 4.
		(d) gives an example illustrating the difference between the two performance measures. (e) and (f) show our results on the same task, but using only 1\% of the available training to show that our algorithm outperforms deep learning-based approaches in data scarce environments.
	}\label{fig:amazon}
\end{figure*}

We use two different measures to compare the various algorithms. The first is the \textbf{Accuracy Score}, which simply counts the number of recommended products that the user indeed ended up purchasing. While this is a sensible measure, it does not explicitly consider the order of the sequence. Therefore, we also consider the \textbf{Sequence Score}, which is a measure based on the Kendall-Tau distance \citep{kendallTau}. In short, this measure counts the number of ordered pairs that appear in both the predicted sequence and the true sequence. Figure \ref{scores} gives an example comparing the two measures.

Figures \ref{amazon-80-acc} and \ref{amazon-80-seq} show the performance of the various algorithms using the accuracy score and sequence score, respectively. These results highlight the importance of adaptivity as the adaptive algorithms consistently outperform their non-adaptive counterparts under both scoring regimes. 
Notice that in both cases, as the number of recommendations increases, our proposed Adaptive Sequence-Greedy policy is outperformed only by the Adaptive Feed Forward Neural Network. Although LSTMs are generally considered better for 
sequence data than vanilla feed-forward networks, we think it is a lack of data that causes them to perform poorly in our experiments.

Another observation, which fits the conventional wisdom, is that deep learning-based approaches can perform well when there is a lot of data. However, when the data is scarce, we see that the Sequence-Greedy based approaches outperform the deep learning-based approaches. Figures \ref{amazon-1-acc} and \ref{amazon-1-seq} simulate a data-scarce environment by using only 1\% of the available data as training data. 
Note that the difference between the adaptive algorithms and their non-adaptive counterparts is less obvious in this setting because the adaptive algorithms use correct guesses to improve future recommendations, but the data scarcity makes it difficult to make a correct guess in the first place.

Aside from competitive accuracy and sequence scores, the Adaptive Sequence-Greedy algorithm provides several advantages over the neural network-based approaches. From a theoretical perspective, the Adaptive Sequence-Greedy algorithm has provable guarantees on its performance, while little is known about the theoretical performance of neural networks. Furthermore, the decisions made by the Adaptive Sequence-Greedy algorithm are easily interpretable and understandable (it is just picking the edge with the highest expected value), while neural networks are generally a black-box. On a similar note, Adaptive Sequence-Greedy may be preferable from an implementation perspective because it does not require any hyperparameter tuning. It is also more robust to changing inputs in the sense that we can easily add another product and its associated edges to our graph, but adding another product to the neural network requires changing the entire input and output structure, and thus, generally necessitates retraining the entire network.

\subsection{Wikipedia Link Prediction} \label{sec:wikipedia}

Using the Wikispeedia dataset~\citep{wikispeedia}, we consider users who are surfing through Wikipedia towards some target article. Given a sequence of articles the user has previously visited, we want to guide her to the page she is trying to reach. Since different pages have different valid links, the order of pages we visit is critical to this task. Formally, given the first $g =3$ pages each user visited, we want to predict which page she is trying to reach by making a series of suggestions for which link to follow. 

In this case, we have $G=(V,E)$, where $V$ is the set of all pages and $E$ is the set of existing links between pages. Similarly to before, the weight $w_{ij}$ of an edge $(i,j) \in E$ is the probability of moving to page $j$ given that the user is currently at page $i$. In this case, there are no self-loops as we assume we can only move using links, and thus we cannot jump to random pages. We again define two states for the nodes: 1 if the user definitely visits this page and 0 if the user does \textit{not} want to visit this page.

This application highlights the importance of adaptivity because the non-adaptive sequence submodularity framework cannot model this problem properly. This is because the Sequence-Greedy algorithm is free to choose any edge in the underlying graph, so there is no way to force the algorithm to pick a link that is connected to the user's current page. On the other hand, with Adaptive Sequence-Greedy, we can use the states to penalize invalid edges, and thus force the algorithm to select only links connected to the user's current page. Similarly, we only have the adaptive versions of the deep learning baselines because we need information about our current page in order to construct a valid path (\cref{deepDetails} gives a more detailed explanation).

Figure \ref{sampleWikiGraph} shows an example of predicted paths, while Figure \ref{wikiScore} shows our quantitative results. More detail about the relevance distance metric is given in \cref{wikiAdditional}, but the idea is that the  it measures the relevance of the final output page to the true target page (a lower score indicates a higher relevance).  The main observation here is that the Adaptive Sequence Greedy algorithm actually outperforms the deep-learning based approaches. The main reason for this discrepancy is likely a lack of data as we have 619 pages to choose from and only 7,399 completed search paths.

\begin{figure*}[htb!]
	\centering
	\subfloat[]{\includegraphics[height=1.44in]{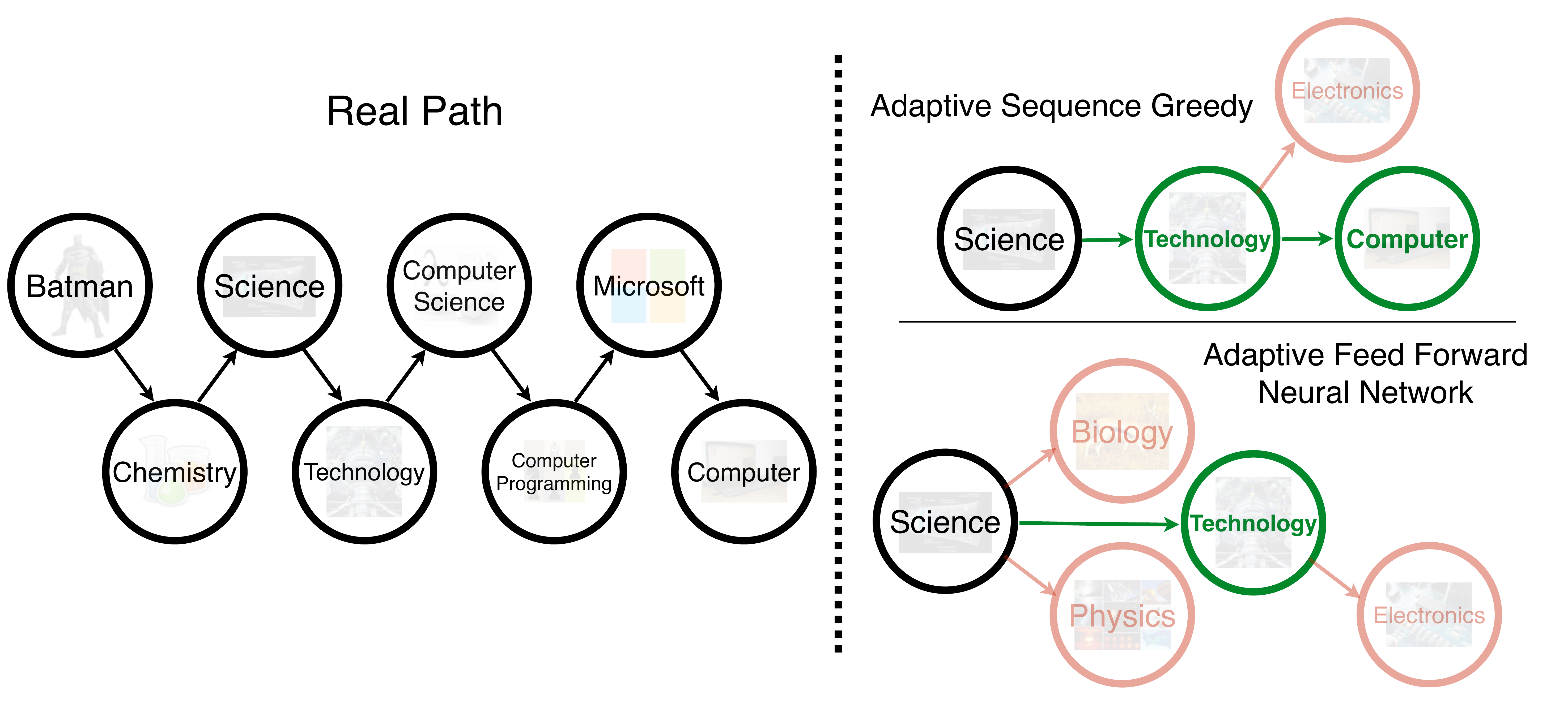}\label{sampleWikiGraph}} \hspace{0.05in}
	\subfloat[]{\includegraphics[height=1.44in]{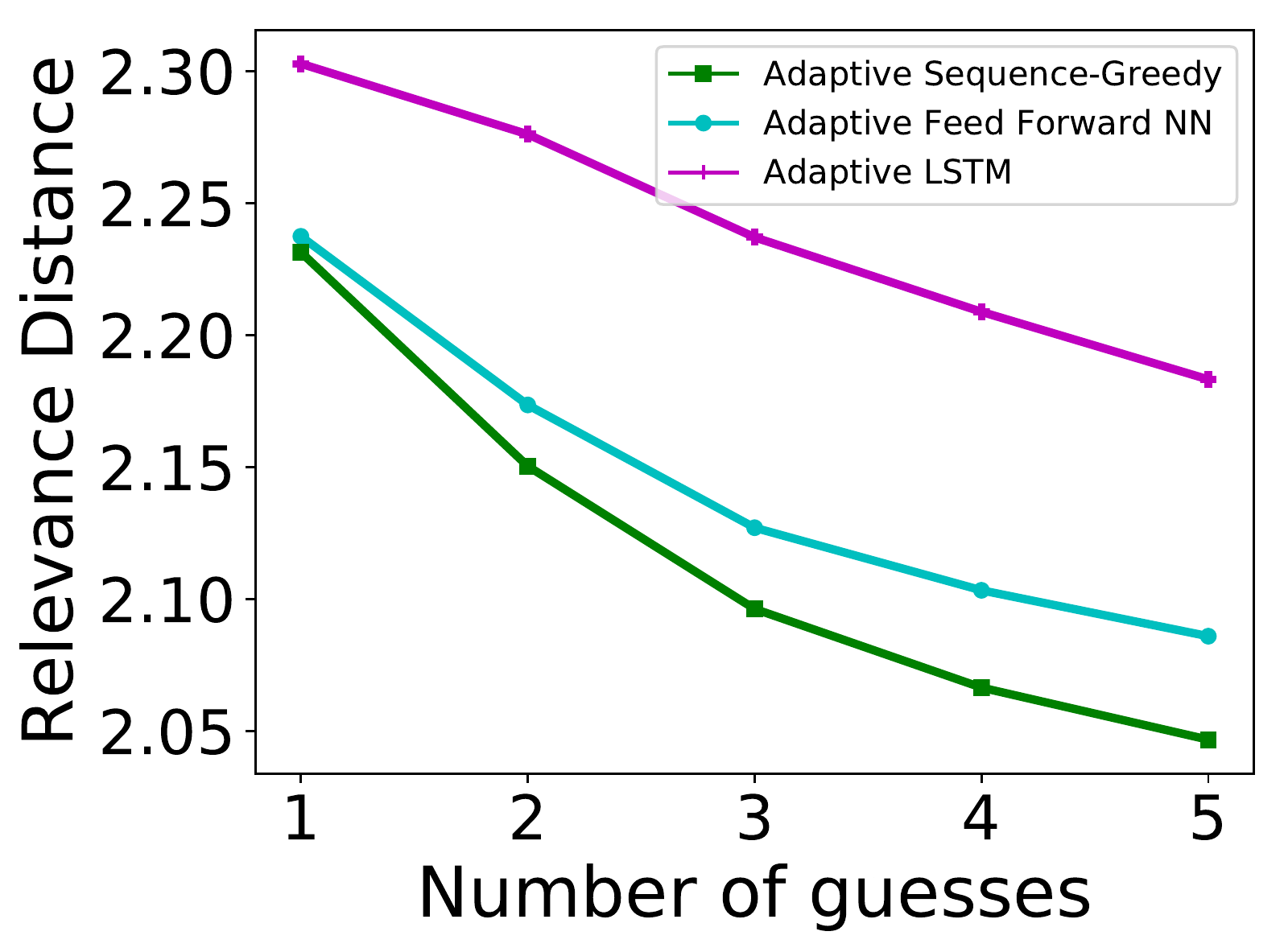}\label{wikiScore}}
	\caption{(a) The left side shows the real path a user followed from \textit{Batman} to \textit{Computer}. Given the first three pages, the right side shows the path predicted by Adaptive Sequence Greedy versus 
		a deep learning-based approach. Green shows correct guesses that were followed, while red shows incorrect guesses that were not pursued further. (b) shows the overall performance of the various approaches.}
	\label{fig:wiki}
\end{figure*}

\section{Conclusion}

In this paper we introduced adaptive sequence submodularity, a general framework for bringing tractability to the broad class of optimization problems that consider both sequences and adaptivity. We presented Adaptive Sequence-Greedy---a general policy for optimizing weakly adaptive sequence submodular functions. 
In addition to providing a provable theoretical guarantee for our algorithm (as well as a discussion about the tightness of this result), we also evaluated its performance on an Amazon product recommendation task and a Wikipedia link prediction task. 
Not only does our Adaptive Sequence-Greedy policy exhibit competitive performance with the state-of-the-art, but it also provides several notable advantages, including interpretability, ease of implementation, and robustness against both data scarcity and input adjustments.

\paragraph{Acknowledgement} Moran Feldman was supported by ISF grant number 1357/16.

\bibliographystyle{plainnat}
\bibliography{adaptiveSubseqMain}

\newpage
\appendix
\section{Table of Notations} \label{sec:notation}
\begin{table}[H]
\centering 
\caption{} \label{table:symbol-table}
		{\small
	\begin{tabular}{|l|p{4.5in}|}
		\hline
		$\ground$ & Ground set of elements. \\
		\hline 
		 $ e \in \ground$ & An individual element from $\ground$. \\
		\hline
		$\phi$ & A realization, i.e., a function from elements to states. \\
		\hline
		$\psi$ & A partial realization to encoding the current set of observations. \\
		\hline
		$\dom(\psi) $ & Domain of a partial realization $\psi$ is defined as $\dom(\psi) = \{ e : \exists o \text{ .s.t. } (o,e) \in \psi \}$.  \\
		\hline
		$\Phi, \Psi$ & A random realization and a random partial
		realization, respectively.\\
		\hline
		$\sim$ & For a realization $\phi$ and a partial realization $\psi$: $\phi \sim \psi$ means $\psi(e) = \phi(e)$ for all $e \in \dom(\psi)$.\\
		\hline
		$p(\phi)$ & The probability distribution on realizations.\\
		\hline
		$p(\phi \mid \psi)$ & The conditional distribution on realizations: $p(\phi \mid \psi) \triangleq \Pr[\Phi = \phi \mid \Phi \sim \psi]$.\\
		\hline
		$\pi$ & A policy, which maps partial realizations to items. \\
		\hline
		$E(\pi, \phi)$ & The set of all edges induced  by $\pi$ when run
		under realization $\phi$. \\
		\hline
			$h$ & An objective function of type $h:2^{\ground} \times O^{\ground} \to \bR_{\geq 0}$.\\
		\hline
		$\Delta(e \mid \psi)$ & The conditional expected marginal benefit of $e$ conditioned on $\psi$. \\
		\hline
		$k$ & Budget on the number of selected items.\\
		\hline
	\end{tabular}
}
\end{table}
\section{Proofs} \label{proofs}
In this section, we prove \cref{theory:adaptive-seq,theory:hyper-adaptive-seq}. Towards this goal, we first state some necessary definitions and notations, and present a few results regarding weakly adaptive submodular functions.

\subsection{Weakly Adaptive Sequence Submodular}

\paragraph{Notation}
The random variable $\Phi$ denotes a random realization with respect to the distribution $p(\Phi = \phi)$ over the items (or equivalently vertices of the graph).\footnote{Note that there is a one to one correspondence between a realization $\phi$ over the vertices and a realization $\phi^E$ over the edges.}
For a set $A$, its partial realization (i.e., items in $A$ and their corresponding states) is shown by $\psi_A = \{ (e, O(e))  \mid e \in A \}$, where $O(e)$ gives the state of $e$. 
For a partial realization $\psi$, we define $\dom(\psi) = \{ e : \exists \ o \text{ s.t. } (o,e) \in \psi \}$. 
We use $\Psi_A$ to denote a random partial realization over a set $A$.
Note that the distribution of random variable $\Psi_A$ is uniquely defined by the distribution of random variable $\Phi$.
A partial realization $\psi$ is consistent with a realization $\phi$ (we write $\phi \sim \psi$) if they are equal, i.e., they are in the same state, everywhere in the domain of $\psi$.
For the ease of notation, we define $h(\psi) \triangleq h(\dom(\psi), O(\psi))$, where $O(\psi)$ is the state of items in the realization $\psi$.
We also define $\ha(A) \triangleq \bE_\Phi(h(A)) \triangleq \bE_\Phi[h(\Phi_A)]$ which is the expected utility of set $A$ (and states of its elements) over all possible realizations of $A$ under the probability distribution $p(\Phi = \phi)$.
We define $\Delta(e \mid  \psi) = \bE_{\Phi \sim \psi} [ h(\Psi_{\{e\}} + \psi) - h(\psi)]$ which is the conditional expected marginal benefit of item $e$ conditioned on having observed the subrealization $\psi$. 
Note that the random variable $\Psi_{\{e\}}$ is the state of item $e$ with respect to the probability distribution $p(\Phi = \phi \mid \Phi \sim \psi)$.
Similarly, we define $\Delta(A \mid  \psi) = \bE_{\Phi \sim \psi} [ h(\Psi_{A} + \psi) - h(\psi)]$ which is the expected marginal gain of set $A$ to the partial realization $\psi$.
Assume $E(\pi_{\phi})$ is the set of edges induced by  the set of items policy $\pi$ selects under the realization $\phi$. 
The expected utility of policy $\pi$ is defined as $\fa(\pi) \triangleq \ha(E(\pi)) = \bE_{\Phi}[h(E(\pi_{\Phi})]$, where the expectation is taken with respect to $p(\Phi = \phi)$.
For a list of all the notations used in the paper refer to \cref{table:symbol-table}  in \cref{sec:notation}.

Next, we restate the definitions for weakly adaptive set submodular and adaptive monotone functions.

\begin{repdefinition}{def:WeakSetAdaptive}
\defWeakSetAdaptive
\end{repdefinition}

\begin{repdefinition}{def:MonotoneAdaptive}
	\defMonotoneAdaptive
\end{repdefinition}

\cref{def:WeakSetAdaptive} is the generalization of both weak submodularity \citep{das2011submodular} and  adaptive submodularity  \citep{golovin2011adaptive} concepts.

Next, we state a few useful claims regarding weakly adaptive submodular functions.

First note  for all $\psi$  and for every set $A \subseteq \ground \setminus \dom(\psi)$, from \cref{def:WeakSetAdaptive} and the fact that $\psi \subseteq \psi$, we have 
	\begin{align} \label{eq:weak-gain}
\Delta(A\mid\psi) \leq \frac{1}{\gamma} \cdot \sum_{e \in A} \Delta(e\mid\psi) .
\end{align}

\begin{lemma} \label{lem:setdiff-marginal-gain}
	For all $\psi$ and $A \subseteq B \subseteq \ground \setminus \dom(\psi)$, we have
	\begin{align*}
	\Delta(B \mid \psi) - \Delta(A \mid \psi) \leq  \frac{1}{\gamma} \cdot \sum_{e \in B \setminus A} \Delta(e \mid \psi).
	\end{align*}
\end{lemma}
\begin{proof}
	We have
	\begin{align*}
	 \Delta(B \mid \psi) - \Delta(A \mid \psi)  &  = \sum{} \Pr[\Psi_A = \psi' \mid \Phi \sim \psi] \cdot \sum{} \Pr[\Psi_{B \setminus A} = \psi'' \mid \Phi \sim \psi + \psi'] \\
	&  \hspace{4.8cm} \cdot  \left( \ha(\psi + \psi' + \psi'') - \ha(\psi + \psi' ) \right)  \\
	&  = \sum{} \Pr[\Psi_A = \psi' \mid \Phi \sim \psi] \cdot  \Delta(B \setminus A \mid \psi + \psi') \leq \frac{1}{\gamma} \cdot \sum_{e \in B \setminus A} \Delta(e \mid \psi),
	\end{align*}
	where the inequality is derived from the definition of weakly adaptive set submodular functions (see \cref{def:WeakSetAdaptive}) and the fact that  $\sum{} \Pr[\Psi_A = \psi' \mid \Phi \sim \psi]  = 1 $.
\end{proof}

\begin{corollary} \label{cor:setdiffsize-marginal-gain}
	For all $\psi$, $e^* = \argmax_{e \in \ground} \Delta(e \mid \psi)$ and two random subsets $A \subseteq B \subseteq \ground \setminus \dom(\psi)$ whose randomness might depend on the realization, we have
	\begin{align*}
	\bE[\Delta(B \mid \psi) - \Delta(A \mid \psi) \mid \Phi \sim \psi] \leq  \frac{\bE[|B \setminus A| \mid \Phi \sim \psi]}{\gamma} \cdot  \Delta(e^* \mid \psi).
	\end{align*}
\end{corollary}
\begin{proof}
By taking expectation over the guarantee of \cref{lem:setdiff-marginal-gain}, we get
\begin{align*}
	\bE[\Delta(B \mid \psi) - \Delta(A \mid \psi) \mid \Phi \sim \psi]
	\leq{} &
	\frac{1}{\gamma} \cdot \bE\left[\sum_{e \in B \setminus A} \Delta(e \mid \psi) \mid \Phi \sim \psi\right]\\
	\leq{} &
	\frac{1}{\gamma} \cdot \bE\left[\sum_{e \in B \setminus A} \Delta(e^* \mid \psi) \mid \Phi \sim \psi\right]
	\\
	={} &  \frac{\bE[|B \setminus A| \mid \Phi \sim \psi]}{\gamma} \cdot  \Delta(e^* \mid \psi),
\end{align*}
where the second inequality follows from the fact that $e^*$ is the element with the largest expected gain.
\end{proof}

The following observation is an immediate consequence of \cref{def:MonotoneAdaptive}.
\begin{observation} \label{obs:monotone}
	For any two (possibly random) subsets $A \subseteq B \subseteq \ground$, we have
	\begin{align*}
	\bE_\Phi(h(A)) \leq \bE_\Phi(h(B)).
	\end{align*}
\end{observation}

\begin{lemma}
	Assume $h$ is adaptive monotone and weakly adaptive 
	set submodular with a parameter $\gamma$ with respect to the distribution $p(\phi)$, and $\pi$ is 
	a greedy policy which picks the item with the largest expected marginal gain at each step, then for all policies $\pi^{*}$ we have 
	\[ \ha(\pi) \geq \left(1 - e^{-\nicefrac{1}{\gamma}}\right) \cdot \ha(\pi^*). \]
\end{lemma}
\begin{proof} The proof of this lemma follows the same line of argument as the proof of \citep[Theroem~5]{golovin2011adaptive}.
\end{proof}

\subsection{Proof of \cref{theory:adaptive-seq}} \label{app:graphs}
In this section, we first restate \cref{theory:adaptive-seq} and then prove it.
\begin{reptheorem}{theory:adaptive-seq}
	\TheoryAdaptiveSeq
\end{reptheorem}

We assume the function $h$ is weakly adaptive set submodular (with a parameter $\gamma$) and monotone adaptive submodular. 
Furthermore, we assume $\pi^*$ is the optimal policy. It means $\pi^*$ maximizes the expected gain over the distribution $\Phi$.

Let $\ell = \lceil k / 2 \rceil$. 
For every $0 \leq s \leq \ell$, let $\pi_s$ be the set of items picked by the greedy policy $\pi$ after $s$ iterations (if the algorithm does not make that many iterations because the set $\cE$ became empty at some earlier point, then we assume for the sake of the proof that the algorithm continues to make dummy iterations after the point in which $\cE$ becomes empty, and in the dummy iterations it picks no items).
The observed partial realization of edges after $s$ iterations of the algorithm is represented by $\psi_{s}$.
The random variable representing $\psi_{s}$ is $\Psi_{s}$.
We define $\fa(\pi_s) \triangleq \ha(E(\pi_s))$, i.e., it is the expected value of items picked by the greedy policy $\pi$ after $s$ iterations.
For every $1 \leq s \leq \ell$, we also denote by $e_s$ and $\cE_s$ the values assigned to the variables $e_{ij}$ and $\cE$, respectively, at iteration number $s$.
Finally, we assume $e_s$ is  a dummy arc with zero marginal contribution to $h$ if iteration number $s$ is a dummy iteration (i.e., the algorithm makes in reality less than $s$ iterations).

\begin{observation} \label{obs:sigma_properties}
	For every $0 \leq s_1 \leq s_2 \leq \ell$, conditioned on the partial realization $\psi_{s_1}$, i.e., the policy has already made its first $s_1$ iterations, we have $\cE_{s_1} \supseteq \cE_{s_2}$ and $E(\pi_{s_1}) \subseteq E(\pi_{s_2})$.
\end{observation}
\begin{proof}
	Both properties guaranteed by the observation follow from the fact that:  for all possible realization $\phi \sim \psi_{s_1}$, we have that $\pi_{s_1}$ is a (possibly trivial) prefix of $\pi_{s_2}$.
\end{proof}

\begin{lemma} \label{lem:gain}
	For every $1 \leq s \leq \ell$, $\fa(\pi_s) - \fa(\pi_{s-1}) \geq \bE_{\Psi_{s- 1}} [ \Delta(e_s \mid \Psi_{s-1} )].$
\end{lemma}
\begin{proof}
Consider a fixed sub-realization $\psi_{s-1}$.
If $e_s$ is a dummy arc, then $\pi_s = \pi_{s - 1}$, and the observation is trivial. Otherwise, notice that the membership of $e_s$ in $\cE_{s - 1}$ guarantees that it does not belong to $E(\pi_{s - 1}) = \dom(\psi_{s-1})$, but does belong to $E(\sigma_s)$.
 Together with the fact that $E(\pi_{s - 1}) \subseteq E(\pi_s)$ by Observation~\ref{obs:sigma_properties}, we get $E(\pi_{s - 1}) + e_s \subseteq E(\pi_s)$; which implies, by the adaptive monotonicity of $h$,
\begin{align*}
\fa(\pi_s) - f(\pi_{s - 1})
& =
\bE_{\Phi \sim \psi_{s-1}} [\fa(\pi_s) ] - h(\psi_{s-1} )  \\
& \geq
\bE_{\Phi \sim \psi_{s-1}} [h(\psi_{s-1} + e_s) ] - h(\psi_{s-1} )  \\
& =
\Delta(e_s \mid \psi_{s-1}).
\qedhere
\end{align*}
Note that we condition on the fact that the first $s - 1$ steps of the policy $\pi$ are performed, therefore we have $\fa(\pi_{s-1}) = h(\psi_{s-1})$.
By taking expectation over all the possible realizations of the random variable $\Psi_{s-1}$ the lemma is proven.
\end{proof}

\begin{lemma} \label{lem:size}
	Conditioned on any arbitrary partial realization $\psi$, we have $\bE_{\Phi \sim \psi}[|E(\pi^*)|] \leq (k-1) d_{\textup{in}}$.
\end{lemma}
\begin{proof}
	The optimal policy under each realization of the random variable $\Phi$ chooses at most $k$ items. Each one of these $k$ items (except the first one) will have at most $d_{\textrm{in}}$ incoming edges. Therefore, the expected number of edges is at most $(k-1) d_{\textrm{in}}$.	
\end{proof}
	
	\begin{lemma} \label{lem:loss}
		For every $1 \leq s \leq \ell$, we have
		\begin{align*}
		\bE_\Phi[h((E&(\pi^*)  \cap \cE_{s - 1}) \cup E(\pi_{s-1}))] 
		\leq \\
	&	\bE_\Phi[h((E(\pi^*) \cap \cE_s) \cup E(\pi_s))] + \dfrac{1}{\gamma} \cdot \bE_\Phi[|E(\pi^*) \cap (\cE_{s-1} \setminus \cE_s)| \cdot \Delta(e_s \mid E(\pi_{s - 1}))].
		\end{align*}
		Note that the expectation is taken over all the possible realizations of the random variable $\Phi$.
	\end{lemma}

	\begin{proof}
	
		The lemma follows by combining the two inequalities of \cref{eq:expec-marignal} and  \cref{eq:expec-diff}.
		\begin{align} \label{eq:expec-marignal}
		&\bE_{\Phi}[\Delta(E(\pi^*) \cap \cE_{s-1} \mid E(\pi_{s-1}))] - \bE_{\Phi}[\Delta(E(\pi^*) \cap \cE_{s} \mid E(\pi_{s-1}))] \\
		&  = \sum \pr[\Psi_{s-1} = \psi_{s-1}]  \cdot 
		\left[ \bE_{\Phi \sim \psi_{s-1}}[  \Delta(E(\pi^*) \cap \cE_{s-1} \mid \psi_{s-1} ) -\Delta(E(\pi^*) \cap \cE_{s} \mid \psi_{s-1}) ] \right] \nonumber \\
		& \overset{(a)}{\leq}  \frac{1}{\gamma} \sum \pr[\Psi_{s-1} = \psi_{s-1}] \cdot \bE_{\Phi \sim \psi_{s-1}}[|E(\pi^*) \cap (\cE_{s-1} \setminus \cE_s)| \cdot \Delta(e_s \mid \psi_{s-1})]  \nonumber \\
		& =  \dfrac{1}{\gamma} \cdot \bE_\Phi[|E(\pi^*) \cap (\cE_{s-1} \setminus \cE_s)| \cdot \Delta(e_s \mid E(\pi_{s - 1}))].\nonumber
		\end{align}
		To see why inequality $(a)$ is true, note that for every given sub realization $\psi_{s-1}$  we have: (i)
		if $e_s$ is a dummy edge, then $(E(\pi^*) \cap \cE_{s - 1}) \cup E(\pi_{s-1}) = \left(E(\pi^*) \cap \cE_s\right) \cup E(\pi_s)$, which makes $(a)$ trivial,
		or (ii) when $e_s$ is not dummy, $(a)$ results from \cref{cor:setdiffsize-marginal-gain}.

		\begin{align}\label{eq:expec-diff}
			&\bE_\Phi[h((E(\pi^*) \cap \cE_{s - 1}) \cup E(\pi_{s-1}))] -  \bE_\Phi[h((E(\pi^*) \cap \cE_s) \cup E(\pi_s))]  \\
		& \leq \bE_\Phi[h((E(\pi^*) \cap \cE_{s - 1}) \cup E(\pi_{s-1}))]- \bE_\Phi[h((E(\pi^*) \cap \cE_s) \cup E(\pi_{s-1}))])  \nonumber \\
		& =  \sum \pr[\Psi_{s-1} = \psi_{s-1}] \cdot  \bE_{\Phi \sim \psi_{s-1}}  [h((E(\pi^*) \cap \cE_{s-1}) \cup \psi_{s-1}) - h(E(\pi^*) \cap \cE_{s}) \cup \psi_{s-1}) ]  \nonumber \\
		& =  \sum \pr[\Psi_{s-1} = \psi_{s-1}] \cdot
		 \bE_{\Phi \sim \psi_{s-1}} \left[ \Delta((E(\pi^*) \cap \cE_{s-1}) \mid \psi_{s-1}) - \Delta(E(\pi^*) \cap \cE_{s}) \mid \psi_{s-1}) \right] \nonumber \\
		 & = \bE_{\Phi}[ \Delta(E(\pi^*) \cap \cE_{s-1} \mid E(\pi_{s-1}))] - \bE_{\Phi}[ \Delta(E(\pi^*) \cap \cE_{s} \mid E(\pi_{s-1}))]. 
		\qedhere
		\end{align} 
	\end{proof}
	
	\begin{lemma} \label{lem:opt-greed-diff}
	$\displaystyle \bE_{\Phi}[ h((E (\pi^*) \cap  \cE_{\ell}) \cup E(\pi_\ell))]  ]\leq \dfrac{1}{\gamma} \cdot
\bE_{\Phi}[	|E(\pi^*) \cap \cE_{\ell}|	 \cdot \Delta(e_\ell \mid \Psi_{\ell-1})] + \fa(\pi_{\ell}).$
	\end{lemma}
	\begin{proof}
	We have
		\begin{align*}
		 \bE_{\Phi}[h((E(\pi^*) \cap \cE_{\ell}) & \cup E(\pi_\ell)) -  h(\pi_{\ell})] 
		\\
		& =
		\sum \pr[\Psi_{\ell} = \psi_{\ell}]  \cdot \bE_{\Phi \sim \psi_{\ell} } [ h( E(\pi^*) \cap \cE_{\ell}) \cup \psi_{\ell}) - h(\psi_\ell) ] 
		 \\
		& \overset{(a)}{\leq}   \frac{1}{\gamma}  \sum \pr[\Psi_{\ell} =
		 \psi_{\ell}] \cdot \bE_{\Phi \sim \psi_{\ell}}[	|E(\pi^*) \cap \cE_{\ell}|	 \cdot \Delta(e_\ell \mid \psi_{\ell-1})] \\
		 & =  \frac{1}{\gamma} \cdot \bE_{\Phi}[	|E(\pi^*) \cap \cE_{\ell}|	 \cdot  \Delta(e_\ell \mid \Psi_{\ell-1})].
		\end{align*}
	To see why inequality $(a)$ is true, note that for every given sub realization $\psi_{\ell}$  we have: (i) if $e_\ell$ is a dummy edge, then $\cE_\ell = \varnothing$, which makes inequality $(a)$ trivial, and (ii) if $e_\ell$ is not a dummy edge then we conclude  inequality $(a)$ from the definition of weakly adaptive set submodular functions (see \cref{def:WeakSetAdaptive}).

The lemma follows by combining this inequality with the observation that $\fa(\pi_{\ell}) = \bE_{\Phi}[h(\pi_{\ell})]$.
	\end{proof}

To combine the last two lemmata, we need the following observation.
\begin{observation} \label{obs:terms_monotonicity}
	For every $2 \leq s \leq \ell$, $\bE_{\Phi}[\Delta(e_{s-1} \mid E(\pi_{s-2}))] \geq \gamma \cdot \bE_{\Phi}[\Delta(e_s \mid E(\pi_{s - 1}))]$.
\end{observation}

We are now ready to prove Theorem~\ref{theory:adaptive-seq}.

\begin{proof}[Proof of Theorem~\ref{theory:adaptive-seq}]
	Combining Lemmata~\ref{lem:loss} and~\ref{lem:opt-greed-diff}, we get
	{\allowdisplaybreaks
	\begin{align} \label{eq:basic_bound}
	\fa(\pi^*) -    &  \fa(\pi_\ell)
 	= 
\bE_{\Phi}[ h((E(\pi^*) \cap \cE_1) \cup E(\pi_0))] - \fa(\pi_\ell) \nonumber \\ 
\nonumber
\leq{} &
\dfrac{1}{\gamma} \cdot \sum_{s = 1}^\ell \bE_\Phi[|E(\pi^*) \cap (\cE_{s-1} \setminus \cE_s)| \cdot \Delta(e_s \mid E(\pi_{s - 1}))] \\
\nonumber
& \hspace{3.5cm}  +
\bE_\Phi[\Delta((E(\pi^*) \cap \cE_s) \cup E(\pi_s))] - \fa(\pi_\ell)  \\
\nonumber
\leq & \dfrac{1}{\gamma}  \sum_{s = 1}^\ell \bE_\Phi[|E(\pi^*) \cap (\cE_{s-1} \setminus \cE_s)| \cdot \Delta(e_s \mid E(\pi_{s - 1}))] \\
\nonumber
& \hspace{4.5cm} + \dfrac{1}{\gamma} \cdot
\bE_{\Phi}[	|E(\pi^*) \cap \cE_{\ell}|	 \cdot \Delta(e_\ell \mid \Psi_{\ell-1})]
\\ 
=\nonumber &
\dfrac{1}{\gamma} \cdot
\sum_{s = 1}^{\ell - 1} \bE_\Phi[|E(\pi^*) \cap (\cE_0 \setminus \cE_s)| \cdot [\Delta(e_s \mid  E(\pi_{s - 1})) - \Delta(e_{s + 1} \mid E(\pi_s))]]
\\  & 
\hspace{4.1cm} +
\dfrac{1}{\gamma} \cdot
\bE_\Phi[|E(\pi^*) \cap \cE_0| \cdot \Delta(e_\ell \mid E(\pi_{\ell - 1}))], 
	\end{align}
	}%
	where the first equality holds since the fact that $\sigma_0$ is an empty sequence implies $E(\sigma_0) = \varnothing$ and $\cE_0 = E$, and the second equality holds since $\cE_{s} \subseteq \cE_{s-1}$ by Observation~\ref{obs:sigma_properties} for every $1 \leq s \leq \ell$.
	We now observe that for every $1 \leq s \leq \ell$, $\pi_s$ contains at most $2s$ vertices. Since each one of these vertices can be the end point of at most $\InDegree$ arcs, we get
	\[
	|E(\sigma^*) \cap (\cE_0 \setminus \cE_s)|
	\leq
	|\cE_0 \setminus \cE_s|
	\leq
	2s\InDegree
	\]
	Additionally, by \cref{lem:size},
	\[
	|E(\sigma^*) \cap \cE_0|
	\leq
	|E(\sigma^*)|
	\leq
	(k-1)\InDegree
	\leq
	2\ell\InDegree.
	\]
	
	Plugging the last two inequalities into Inequality~\eqref{eq:basic_bound} yields
	\begin{align*}
	\fa(\pi^*) - \fa(\pi_\ell)
	\leq{} &
	\sum_{s = 1}^{\ell - 1} 	\dfrac{2s\InDegree}{\gamma} \cdot
	\bE_{\Phi}[\Delta(e_s \mid E(\pi_{s - 1})) - \Delta(e_{s + 1} \mid E(\pi_s))]
	+ \\
& \hspace{5.9cm}	\dfrac{2\ell\InDegree }{\gamma} \cdot 
\bE_{\Phi}[\Delta(e_\ell \mid E(\sigma_{\ell - 1}))]\\
	={} &
	\sum_{s = 1}^{\ell} \dfrac{2\InDegree}{\gamma} \cdot \bE_{\Phi}[\Delta(e_s \mid E(\pi_{s - 1}))]
	\leq
	\dfrac{2\InDegree}{\gamma} \cdot \sum_{s = 1}^{\ell} [\fa(\pi_s) - \fa(\pi_{s-1})]\\
	={} &
	\dfrac{2\InDegree}{\gamma} \cdot [\fa(\pi_\ell) - \fa(\pi_0)]
	\leq
	\dfrac{2\InDegree}{\gamma}\cdot \fa(\pi_\ell),
	\end{align*}
	where the second inequality holds due to \cref{lem:gain} and the last inequality follows from the non-negativity of $f$. 
	Rearranging the last inequality, we get
	\[
	\fa(\pi_\ell)
	\geq
	\frac{\gamma}{2\InDegree + \gamma} \cdot \fa(\pi^*),
	\]
	which implies the theorem since $\fa(\pi_\ell)$ is a lower bound on the expected value of the output sequence of  \cref{alg:sequence-greedy} because $\sigma_\ell$ is always a prefix of this sequence.
\end{proof}

\subsection{Proof of \cref{theory:hyper-adaptive-seq}} \label{sec:hypergraph}

In this section, we first restate and then prove \cref{theory:hyper-adaptive-seq} which guarantees the performance of our proposed policy applied to hypergraphs.

\begin{reptheorem}{theory:hyper-adaptive-seq}
	\TheoryAdaptivehyper
\end{reptheorem}

\begin{algorithm}[htb!]
	\caption{Adaptive Hyper Sequence Greedy}\label{alg:hyper-sequence-greedy}
	\begin{algorithmic}[1]
		\STATE	\algorithmicrequire  \ Directed hypergraph $H(V, E)$ ,  $\gamma$-adaptive and adaptive-monotone function $h: 2^{E} \times O^{E} \rightarrow \bR_{\geq 0}$ and cardinality parameter $k$\;
		\STATE	Let $\sigma \gets ()$
		\WHILE{$|\sigma| \leq k - r $}
		{
			\STATE		$\cE = \{e \in E \mid \sigma \cap V(e) \text{ is a prefix of} \  e \} $
			\IF{$\cE \neq \emptyset$}
			{
				\STATE	$e^* = \argmax_{e \in \cE} \Delta(e \mid \psi_{\sigma})$
				\FOR{every v $\in e^*$ in order}
				\IF{$v \notin \sigma$}
				{
					\STATE $\sigma = \sigma \oplus v$
				}
				\ENDIF
				\ENDFOR
				\STATE 		Identify the state of all edges in $\cE' = \{e \in E \mid$ all elements of  $V(e)$  belong to  $\sigma$ and appear in the same order$\}  $
				\STATE 			$\psi_{\sigma} =  \psi_{\cE'}$
			}
			\ELSE
			{
				\STATE \textbf{break}
			}
			\ENDIF
		}
		\ENDWHILE
		\STATE {\bfseries Return} $\sigma$
	\end{algorithmic}
\end{algorithm}

In the proof of this theorem we use the same notation that we used in Section~\ref{app:graphs} for analyzing \cref{alg:sequence-greedy}, with the exception of $\cE_s$, which is now defined as $\cE_s = \{e \in E \mid \sigma_s \cap V(e) \text{ is a prefix of } e\}$, and $\ell$, which is now defined as $\lfloor k/r \rfloor$.

The following lemma is a counterpart of \cref{lem:size}.

\begin{lemma} \label{lem:hyper_size}
	$\displaystyle  |E(\sigma^*)| \leq (k-r +1)\InDegree$.
\end{lemma}
\begin{proof}
	For a realization $\phi$, every arc of $\pi^*$ must end at a vertex of $\pi^*$ which is not one of the first $r - 1$ vertices. The observation follows since $\pi^*$ contains at most $k - r + 1$ vertices of this kind, and at most $\InDegree$ arcs can end at each one of them.
\end{proof}

One can observe that the proofs of all the other observations and lemmata of Section~\ref{app:graphs} are unaffected by the differences between \cref{alg:sequence-greedy} and \cref{alg:hyper-sequence-greedy}, and thus, these observations and lemmata can be used towards the proof of \cref{theory:hyper-adaptive-seq}.

\begin{proof}[Proof of \cref{theory:hyper-adaptive-seq}]
	The proof of this theorem is identical to the proof of \cref{theory:adaptive-seq} up to two changes. 
	First, instead of getting an upper bound of $2s\InDegree$ on $|\cE_0 \setminus \cE_s|$ for every $1 \leq s \leq \ell$, we now get an upper bound of $rs\InDegree$ on this expression because $\sigma_s$ might contain up to $rs$ vertices rather than only $2s$. Second, instead of getting an upper bound of $2\ell\InDegree$ on $|E(\sigma^*)|$, we now use \cref{lem:hyper_size} to get an upper bound of $(k - r + 1)\InDegree \leq r\ell\InDegree$ on this expression.
\end{proof}

\section{Proof of \cref{theory:hardness}} \label{app:hardness}
  The approximability of the sequence submodular maximization, as a generalization of the densest $k$ subgraph problem (DkS) \citep{kortsarz1993on}, is an open theoretical question with important implications. In this section, we prove \cref{theory:hardness}.
  
In the DkS problem the goal is to find a subgraph on exactly $k$ vertices that contains the maximum number of edges. 
DkS as a generalization of the $k$-clique problem is NP-hard and the best polynomial algorithm for DkS achieves a $n^{1/4 + \epsilon}$ approximation factor\footnote{Note that in this section we define the approximation factor as the ratio of the the optimal solution to the solution provided by the algorithm.}  for an arbitrary $\epsilon > 0$  \citep{bhaskara2010detecting}.
Furthermore, there exists no polynomial time algorithm that approximates DkS within an $O(n^{1/(\log \log n)^c})$ factor unless $3$-SAT has a subexponential time algorithm \cite{manurangsi2017almost}.

\begin{lemma} \label{lemma:seq-DkS}
	Any algorithm with an $\alpha$ approximation factor to the sequence submodular maximization problem solves the densest $k$ subgraph problem (DkS) with at most  an $\alpha$ approximation factor.
\end{lemma}
\begin{proof}
	To prove this lemma, we show that for each instance of DkS over a directed graph $G(V,E)$ we can build an instance of the sequence submodular maximization problem over a directed graph $H(V,E')$ such that solving the latter problem also solves the former one. 
		We assume all vertices and edges have a single state. Therefore, the problem translates to the non-adaptive sequence submodular scenario.
		
	Graph $H$ is built from graph $G$ by replacing each edge $e = (u,v)$ in $E$ by two directed edges $(u,v)$ and $(v,u).$ 
	We define $h(S) = |S|$, which is linear and therefore submodular. 
	Finally, the sequence submodular function $f$ is defined as $f(\sigma) = h(E(\sigma)) = |E(\sigma)|$.
	It remains to show that for every subset of vertices $S$ the value of function $f$ for an arbitrary permutation $\sigma_S$ of $S$ is equivalent to the size of subgraph $G_S$  induced  by those vertices in graph $G$.
	This is true because for every edge $(u,v) \in G_S$ we have two corresponding edges in the directed graph $H$ and based on the order of $u$ and $v$ exactly one of them is considered in $E(\sigma_S)$.
	
As a result, maximizing the function $f$ with a cardinality constraint $k$ is equivalent to solving the DkS problem. Thus, any algorithm with an $\alpha$ approximation factor to the sequence submodular maximization problem solves DkS with at least  an $\alpha$ approximation factor.
\end{proof}

\citet{manurangsi2017almost} showed that any algorithm with an $O(n^{1/(\log \log n)^c})$ approximation factor to the DkS problem (for a constant $c > 0$) would prove the exponential time hypothesis is false. Next, we directly state the result of \cite{manurangsi2017almost}.

\begin{theorem}[\citet{manurangsi2017almost}, Theorem~1] \label{theo:DkS}
	There is a constant $c > 0$ such that, assuming the exponential time hypothesis, no polynomial-time algorithm can, given
	a graph $G$ on $n$ vertices and a positive integer $k \leq n$, distinguish between the following two cases:
	\begin{itemize}
		\item There exist $k$ vertices of $G$ that induce a $k$-clique.
		\item Every $k$-subgraph of $G$ has density at most $n^{-1/(\log \log n)^c}$.
	\end{itemize}
\end{theorem}

To sum-up, \cref{theory:hardness} is proved from the combination of the two following facts: 
\begin{enumerate}
	\item If there is an algorithm with an approximation within a $n^{1/(\log \log n)^c}$ factor to the sequence submodular maximization problem, from the result of \cref{lem:hyper_size}, we know that it would solve the DkS problem with at most the same factor. 
	\item If there is an algorithm with a $n^{1/(\log \log n)^c}$ approximation factor to the DkS problem, it could distinguish the two cases of \cref{theo:DkS} and would prove the exponential time hypothesis to be false.
\end{enumerate}

\section{Additional Experimental Details} \label{baselineDetails}

\subsection{Amazon Product Recommendation} \label{amazonAdditional}

In this application, we consider the task of recommending products to users. In particular, we use the Amazon Video Games review dataset \citep{amazonReviews}, which contains 10,672 products, 24,303 users, and 231,780 confirmed purchases. We furthered focused on the products that had been purchased at least 50 times each, leaving us with a total of 958 unique products.

Although we are using a different dataset, the experimental set-up closely follows that of the movie recommendation task in \citet{tschiatschek17} and \citet{mitrovic18a}. We first group and sort all the data so that each user $u$ has an associated sequence $\sigma_u$ of products that they have purchased. These user sequences are then randomly partitioned into a training set and a testing set using a 80/20 split. Note that we 5 trials to average our results.

\vspace{0.03in}
Using the training set, we build a graph $G = (V,E)$, where $V$ is the set of all products and $E$ is the set of edges between these products. Each product $i \in V$ has a self-loop $(i,i)$, where the weight (denoted $w_{ii}$) is the fraction of users in the training set that purchased product $v_i$. Similarly, for each edge $(i,j)$, the corresponding weight $w_{ij}$ is defined to be the conditional probability of purchasing product $j$ given that the user has previously purchased product $i$.

\vspace{0.03in}
For each sequence $\sigma_u$ in the test set, we are given the first $g$ products that user $u$ purchased, and then we want to predict the next $k$ products that she will purchase. 
After each product is recommended to the user, the state of the product is revealed to be 1 if the user has indeed purchased that product, and 0 otherwise. At the start, the $g$ given products are known to be in state 1, while the states of the remaining products are initially unknown. 

\vspace{0.03in}
As described in \cref{sec:ass}, the states of the edges are determined by the states of the nodes. In this case, the state of each edge $(i,j)$ is equal to the state of product $i$. The intuitive idea is that edge $(i,j)$ encodes the value of purchasing product $j$ after already having purchased product $i$. Therefore, if the user has definitely purchased product $i$ (i.e., product $i$ is in state 1), then they should receive the full value of $w_{ij}$. On the other hand, if she has definitely not purchased product $i$ (i.e., product $i$ is in state 0), then edge $(i,j)$ provides no value. Lastly, if the state of product $i$ is unknown, then the expected gain of edge $(i,j)$ is discounted by $w_{ii}$, the value of the self-loop on $i$, which can be viewed as a simple estimate for the probability of the user purchasing product $i$. See Figure \ref{amazonGraph} for a small example.

We use a probabilistic coverage utility function as our monotone adaptive submodular function $h$. Mathematically, 
\[
h(E_1) = \sum_{j \in V} \Big[ 1 - \prod_{(i,j) \in E_1} (1 - w_{ij}) \Big],
\]
where $E_1 \subseteq E$ is the subset of edges that are in state 1.

\subsection{Wikipedia Link Prediction} \label{wikiAdditional}

We use the Wikispeedia dataset~\citep{wikispeedia}, which consists of 51,138 completed search paths on a condensed version of Wikipedia that contains 4,604 pages and 119,882 links between them. We further condense the dataset to include only articles that have been visited at least 100 times, leaving us with 619 unique pages and 7,399 completed search paths.

One natural idea for scoring each algorithm would be to look at the length of the shortest path between the predicted target and the true target. However, the problem with this metric is that all the popular pages have relatively short paths to most potential targets (primarily since they have so many available links to begin with). Hence, under this scoring, just choosing a popular page like ``Earth" would be competitive with many more involved algorithms. 

Instead, we define a measure we call the \textit{Relevance Distance}. The relevance distance of a page $i$ to a target page $j$ is calculated by taking the average shortest path length to $j$ across all neighboring pages of $i$. A lower distance indicates a higher relevance. For example, if our target page is \textit{Computer Science}, both \textit{Earth $\rightarrow$ Earth Science $\rightarrow$ Computer Science} and \textit{University $\rightarrow$ Education $\rightarrow$ Computer Science} have a shortest path of length 2. However, the relevance distance of \textit{Earth} to \textit{Computer Science} is 2.68, while the relevance distance of \textit{University} to \textit{Computer Science} is 2.41, which fits better with the intuition that \textit{University} is logically closer to \textit{Computer Science}.

\subsection{Deep Learning Baseline Details} \label{deepDetails}

\subsubsection{Feed Forward Neural Network}

For both experiments, the input to the Feed Forward Neural Network is a size $|V|$ vector $X$. That is, there is one input for each item in the ground set. In the Amazon product recommendation task in \cref{sec:amazon}, $X_i = 1$ if the user is known to have purchased product $i$ and 0 otherwise. Similarly, for the Wikipedia link prediction task in \cref{sec:wikipedia}, $X_i = 1$ if the user is known to have visited page $i$ and 0 otherwise. 

The output in both cases is a size $|V|$ soft-maxed vector $Y$. In \cref{sec:amazon}, $Y_i$ can be viewed as the probability that product $i$ will be the user's next purchase. In \cref{sec:wikipedia}, $Y_i$ can be viewed as the probability that user will visit page $i$ next.

For the Amazon product recommendation task in \cref{sec:amazon}, each user $u$ in the training set has an associated sequence $\sigma_u$ of products she purchased. Each such sequence was split into $| \sigma_u | - 2$ training points by taking the first $g$ products as input and the $(g+1)$-th product as the output for $g = 1, \hdots, | \sigma_u | - 1$. For each user $u$ in the testing set, we would take the first $g=4$ products she purchased and encode them in the vector $X$ as described above. We would then input this vector into our trained network and output the vector $Y$. In the non-adaptive case we cannot get any feedback from the user, so we simply output the products corresponding to the $k$ highest values in $Y$. 

In the adaptive case, we would look at the largest value $Y_j$ in our output vector and output this as our first recommendation. We then check if the corresponding product appeared somewhere later in the user's sequence $\sigma_u$. If yes, then we would update our input $X$ so that $X_j = 1$ and re-run the network to get our next recommendation. If not, we would simply use the next highest value in $Y_j$ as our next recommendation (since the input doesn't change). This was repeated for $k$ recommendations. This is supposed to mimic interaction with the user where we would recommend a product, and then see whether or not the user actually purchases this product. Note that we only considered values $Y_j$ such that $X_j = 0$ because we did not want to recommend products that we knew the user had already purchased.

The main difference for the Wikipedia task in \cref{sec:wikipedia} is that, in the testing phase, we cannot simply output the top $k$ values in $Y$ as we did above because they likely will not constitute a valid path. Instead, we only have an adaptive version that is similar to what was described above. We find the highest value $Y_j$ such that $X_j = 0$ (i.e. the user had not already been to this page) and a link to page $j$ actually exists from our current page. We output this page $j$ as our recommendation for the user's next page. We then check if the user actually visited our predicted page $j$ at some point in their sequence of pages. If yes, we would update $X$ so that $X_j = 1$ and re-run the network. If not we would look to the next highest value in the output $Y$. This was repeated for $k$ guesses. Note that if we reached the true target page, we would stop making guesses.

In terms of architecture, we used a single hidden layer of 256 nodes with ReLU activations. We use a batch size of 1024 at first and then go down to a batch size of 32 when we are in the low data regime (i.e. only using 1\% of the available training data). We used an 80/20 training/validation split to guide our early stopping criterion during training (with minimum improvement of 0.01 and patience of 1). We used categorical cross-entropy as our loss function.

\subsubsection{LSTM}

The main difference between the LSTM and the feed forward network is in the input. The input to the LSTM is a sequence of one-hot encoded vectors instead of just a single vector. That is, for the LSTM, each vector in the sequence had exactly one index with value $1$.

We experimented with using a long sequence of input vectors and padding with all-zero vectors, but we found better results using a fixed small sequence length $g$ and then ``pushing" the sequence back when updating. For example, if our current input was a sequence of vectors $[ v_1, v_2, v_3 ]$ and we wanted to update it with a new vector $v_4$, the updated input would be $[ v_2, v_3, v_4 ]$.

The adaptive LSTM followed the same set-up as the non-adaptive LSTM, but with the same adaptive update rules described above for the feed-forward neural network.

For all experiments, we used a single hidden layer of 8 LSTM nodes. The other hyperparameters are all the same as described for the Feed Forward network above, except we start at a batch size of 256 instead of 1024 (before also going down to a batch size of 32 in the low data regime).

\end{document}